%% file: structure_learning_2nd_revised.tex
\documentclass[10pt,journal,compsoc]{IEEEtran}
\usepackage{amsmath, amssymb, amsthm, amsfonts}
\usepackage[linkcolor=blue,colorlinks=true,citecolor=blue]{hyperref}
\usepackage{graphicx}
\usepackage[caption=false,font=footnotesize]{subfig}
\usepackage{multirow}

\usepackage{tikz}
\usetikzlibrary{arrows}

\newcommand{\machine}[1]{$\mathcal{M}_{#1}$}
\newcommand{\vect}[1]{\mathbf{#1}}

\newtheorem{theorem}{Theorem}
\newtheorem{definition} {Definition}
\newtheorem{lemma}{Lemma}

\newtheorem{remark}{Remark}

\tikzstyle{vertex}=[circle,fill=black,minimum size=5pt,inner sep=0pt]
\tikzstyle{selected vertex} = [vertex, fill=red!24]
\tikzstyle{edge} = [draw,thick,-,black!40,line width=2pt]
\tikzstyle{weight} = [font=\small]
\tikzstyle{selected edge} = [draw,line width=1pt,-,black]
\tikzstyle{ignored edge} = [draw,line width=5pt,-,black!20]
\tikzstyle{false edge} = [draw,thick,-,red!50,line width=2pt]

\ifCLASSOPTIONcompsoc
  \usepackage[nocompress]{cite}
\else
  \usepackage{cite}
\fi
\ifCLASSINFOpdf
\else
\fi
\begin{document}
%
\title{Learning of Tree-Structured Gaussian Graphical Models on Distributed Data under Communication Constraints}
%
%
%
%

\author{Mostafa~Tavassolipour,
        Seyed~Abolfazl~Motahari,
        and~Mohammad-Taghi~Manzuri~Shalmani
\IEEEcompsocitemizethanks{\IEEEcompsocthanksitem M. Tavassolipour, S. A. Motahari, and M. T. Manzuri Shalmani are with the Department
	of Computer Engineering, Sharif University of Technology, Tehran, Iran.\protect\\
}}

\IEEEtitleabstractindextext{%
\begin{abstract}
In this paper, learning of tree-structured Gaussian graphical models from distributed data is addressed. In our model,  samples are stored in a set of distributed machines where each machine has access to only a subset of features. A central machine is then responsible for learning the structure based on received messages from the other nodes.  We present a set of communication efficient strategies, which are theoretically proved to convey sufficient information for reliable learning of the structure. In particular, our analyses show that even if each machine sends only the signs of its local data samples to the central node, the tree structure can still be recovered with high accuracy. Our simulation results on both synthetic and real-world datasets show that our strategies achieve a desired accuracy in inferring the underlying structure, while spending a small budget on communication. 
\end{abstract}

\begin{IEEEkeywords}
Structure learning, Chow-Liu algorithm, Gaussian Graphical Model.
\end{IEEEkeywords}}

\maketitle

\IEEEdisplaynontitleabstractindextext

%
\IEEEpeerreviewmaketitle

\IEEEraisesectionheading{\section{Introduction}\label{sec:introduction}}
\IEEEPARstart{M}{any} modern systems acquire data at several repositories which are stored at different locations. In many situations, it is impossible to transfer the distributed data completely to a central machine due to communication constraints. Designing communication-efficient learning algorithms is desired to transfer enough information from repositories to the central machine and to reliably infer the learning model.

Many learning algorithms can be modified to run distributively at several machines to perform a learning task. There are many papers that propose distributed (parallel) version of various learning algorithms \cite{boyd2011distributed, forero2010consensus, tavassolipour2017learning,meng2014marginal}. However, some learning algorithms could not be efficiently parallelized on distributed data. For example, when each local machine has access to some attributes (dimensions) of data samples, many learning algorithms could not be  run distributively. In such situations, typically, there exists a central machine which is responsible for running the learning algorithm. Due to communication constraints, local machines could not transmit their whole datasets to the central machine. In fact, the central machine may have access to a lossy compressed version or a subset of the original data. Thus, designing and analysis of the learning algorithms to make an appropriate trade-off between accuracy and the amount of communication is of great importance.


In general, three models can be considered for data in distributed settings: horizontal model, where the data are distributed across samples; vertical model, where data are distributed over dimensions; and hybrid model which is the combination of both horizontal and vertical models. Designing distributed learning systems for the horizontal model has been addressed in many papers \cite{tavassolipour2017learning, jordan2018communication}. In \cite{amari1998statistical} and \cite{el-Gamal2017rate}, the vertical model is studied from information theoretic point of view. This paper, in fact, addresses structure learning of tree-structured Gaussian Graphical Models (GGM) in the case of vertical model.

A GGM is a Markov Random Field (MRF) with normal variables which indeed form a joint normal distribution with mean $\vect{\mu}$ and covariance matrix $Q$. If the $ij$-th component of $Q^{-1}$ is zero, then the variables $i$ and $j$ are conditionally independent given the other variables. From this fact, one can construct the structure of the GGM by connecting node $i$ to $j$ iff the $ij$-th component of $Q^{-1}$ is non-zero. GGMs are widely used in many applications such as gene regulatory networks \cite{irrthum2010inferring, friedman2000using, friedman2004inferring}, brain connectivity learning \cite{huang2010learning}, etc.

In this paper, we analyze and propose structure learning methods based on the Chow-Liu algorithm over distributed data \cite{chow1968approximating}. We assume that the data are split across dimensions (vertical model) among multiple machines. It worths mentioning that, unlike the horizontal model, the local machines are not capable of summarizing any statistics reflecting dependencies between dimensions in the vertical model. Hence, any inference requires some amount of communication between machines. This makes the problem nontrivial and  challenging.
We have assumed that the local machines are connected to a central machine over communication limited channels. Each machine compresses and transmits its local dataset to the central machine. Finally, the central machine by applying the Chow-Liu algorithm on the received distorted data, estimates the structure of the underlying GGM.

Interestingly, we convey an important message regarding accuracy of structure learning on distributed data for tree-structured GGMs: spending few bits per symbol is sufficient to transmit enough information to the central machine for the purpose of estimating the structure with high accuracy. We justify this result by rigorous analysis and numerical experiments on synthetic and real datasets. 

The paper is organized as follows. Section \ref{sec:related work} provides a comprehensive literature review on structure learning and distributed statistical inference. In Section \ref{sec:problem statement}, we describe the problem in great detail and present our main contributions. Section \ref{sec:structure learning using binary} is devoted to structure learning using signs of the Gaussian data with theoretical analyses on its estimation error. In Section \ref{sec:correlation estimation}, a per-symbol quantization  scheme is proposed and analyzed. We present the results of some experiments on real and synthetic datasets in Section~\ref{sec:experiments}. Section ~\ref{sec:conclusion} concludes the paper. 

\section{Related Work} \label{sec:related work}
\subsection{Structure Learning} \label{sec:structure learning methods}

In the context of graphical models, inferring the underlying graph structure from data samples is of great importance. The structure learning is a \emph{model selection} problem in which one estimates the underlying graph from i.i.d. samples  drawn from some MRFs or Bayesian networks. The structure learning plays an important role in many applications such as reconstructing gene regulatory networks from gene expressions \cite{chai2014review, hecker2009gene}, brain connectivity learning \cite{huang2010learning}, relationship analysis in social networks \cite{xiang2010modeling}, etc. 


Structure learning of GGMs is equivalent to recovering the support (fill pattern) of the inverse covariance matrix (concentration matrix). The sparse structure estimation of the concentration matrix is discussed in many works \cite{huang2013sparse,friedman2008sparse, banerjee2008model, danaher2014joint}. Among the sparse methods, maximizing the $\ell_1$-regularized likelihood is the most popular one. \cite{friedman2008sparse} and \cite{banerjee2008model} proposed the graphical lasso (\emph{glasso}), which finds the ML concentration matrix by an $\ell_1$ regularization term. There are some coordinate-descent algorithms for solving the glasso problem \cite{friedman2008sparse,  mazumder2012graphical}. More recently, Hsieh et al. \cite{hsieh2014quic} proposed a new algorithm for solving the glasso problem enjoying super linear convergence rate. The consistency of the estimated graph using the glasso for high dimensional problems is studied in \cite{ravikumar2011high}.


GGMs have an interesting property: one can obtain the neighborhood of each node by solving a linear regression problem for the corresponding variable on other variables. This approach of structure recovering is also known as \emph{neighborhood selection} in the literature. For the sparse structures, combination of $\ell_1$ regularization and the method of neighborhood selection is studied in \cite{meinshausen2006high}. In this method, an $\ell_1$ regularized linear regression problem is solved separately for each node. Such a method may lead to inconsistencies between the inferred neighborhoods. Chen et al. \cite{chen2014selection} proposed some rules to resolve the inconsistencies. 




Tan et al. \cite{tan2011learning} measured the complexity of the tree structures in view of  the Chow-Liu algorithm. They also provided the analysis of the error exponent of the Chow-Liu algorithm on tree-structured GGMs in  \cite{tan2010learning}. Moreover, they discussed the extreme structures which yield the best and worst error exponents. Resembling some features of this paper, the results in \cite{tan2010learning} differ from our results in two major aspects. First, our analysis is non-asymptotic while theirs is asymptotic in the number of samples. Second, we address the distributed version of the problem where only quantized data is available at the central machine. Having limited  access to the original data poses a significant challenge in the design and analysis.   

\subsection{Distributed Statistical Inference}

The early works on distributed parameter estimation mainly focused on the asymptotic analysis of error exponents for given bit rates (see \cite{amari1998statistical} and refs therein). More recently, studies are focused on characterizing the dependence between estimation performance and the communication constraint (see \cite{xu2017information, zhang2013information, duchi2014optimality,garg2014communication,braverman2016communication}). For example, Zhang et. al. \cite{zhang2013information} and Duchi et. al. \cite{duchi2014optimality} obtained some lower bounds on the minimax risks for distributed statistical parameter estimation under a given communication budget. They have studied the problem under single-round (non-interactive) and multiple-round (interactive) communication protocols between the local machines and the central one.  A similar problem is addressed in some other papers such as \cite{luo2005universal} and \cite{xu2017information}. Luo \cite{luo2005universal} showed that if each machine has a single one dimensional sample and transmits only one bit to the central machine, one can achieve the centralized minimax rate up to a constant factor for some specific  problems.  In a more recent work, Xu and Raginsky in \cite{xu2017information} obtained lower bounds on Bayes risk in estimating parameters in a similar distributed setting. They studied the problem under both interactive and non-interactive communication protocols.



Some basic problems in machine learning such as classification, regression, hypothesis testing, etc. in distributed fashion are studied in \cite{raginsky2007learning, amari2011optimal,tavassolipour2017learning}.
Raginsky in \cite{raginsky2007learning} studied the classification and regression problem in distributed settings. He obtained an information-theoretic characterization of achievable predictor performance. He evaluated the results on non-parametric regression with Gaussian noise. The distributed hypothesis testing  is studied by Amari \cite{amari2011optimal}  where a central machine makes decision on the correlation coefficient of two sequences stored in two different machines.

In this paper, we  focus on the problem of distributed tree-structured GGM learning which is not studied previously. This work is similar to the problems studied by Ahlswede \cite{ahlswede1990minimax} and El-gamal \cite{el-Gamal2017rate} due to the fact that each local machine cannot estimate the parameters without any communication with other machines. This is in contrast to  the most of the mentioned works where the local machines can have their own estimate of the underlying parameters. This fact makes the problem challenging as the local machines communicate with the central machine blindly. 


%

\section{Problem Statement and Proposed Methods} \label{sec:problem statement}
We are given $n$ i.i.d. random vectors $\{\vect x^{(1)}, \cdots, \vect x^{(n)}\}$ drawn from a $d$-dimensional zero mean normal distribution $\mathcal{N}(\vect 0, Q)$. Assume that the normal distribution can be factorized according to a tree model $\mathcal{T}=(\mathcal{V}, \mathcal{E})$ where $\mathcal{V}=\{1, \cdots, d\}$ is the set of nodes and $\mathcal{E}$ is the set of edges. Factorization according to $\mathcal{T}$ means that $(Q^{-1})_{jk}\neq 0$ if and only if $(j,k) \in \mathcal{E}$. Our goal is to find the structure of $\mathcal{T}$ in a situation where data is stored in $M$ machines such that each machine possesses some dimensions of the sample vectors. All machines are connected to a central machine via some communication limited links. This limitation makes it impossible for a machine to communicate its local dataset without any distortion to the central machine.

In this paper, we aim at proposing a communication efficient algorithm for estimating the underlying tree structure. In this setting, each machine transmits some information from its local dataset to the central machine which is responsible for estimating the structure from the received data.

Without loss of generality, we assume that the underlying normal distribution has zero mean and unit variance for all dimensions (i.e. $Q_{jj} = 1$).  For convenience, we also assume that the machine \machine{j} contains the $j$-th dimension of the sample vectors. In this way, the number of machines is equal to the dimensionality of the normal distribution (i.e. $M = d$). We denote the $j$-th dimension of $i$-th sample by $x_j^{(i)}$, i.e. $\vect x^{(i)} = [x^{(i)}_1, \cdots, x^{(i)}_d]^T$. Therefore, the local dataset on machine \machine{j} is $\{x^{(1)}_j, \cdots, x^{(n)}_j\}$. Throughout the paper, we denote the quantized (compressed) version of $x_j^{(i)}$ by $u_j^{(i)}$.

We assume that the communication budget is $R$ bits for each $x_j^{(i)}$. Thus, the overall communication cost for transmitting local datasets to the central machine is $n d R$ bits. 

The overall block diagram of our system is depicted in \figurename{}~\ref{fig:whole system}. In this system, \machine{j} encodes (quantizes) its local dataset using an $R$-bit encoder that can be represented by a function $\psi_j$ which maps the samples to one of the predefined  $2^{nR}$ reconstruction points denoted by $(u_j^{(1)},\cdots,u_j^{(n)})$ , i.e.
$
(u_j^{(1)},\cdots,u_j^{(n)}) = \psi_j(x_j^{(1)},\cdots,x_j^{(n)}).
$
Since $(u_j^{(1)},\cdots,u_j^{(n)})$ can take only one of the $2^{nR}$ different reconstruction points, it can be transmitted to the central machine with $nR$ bits. We assume that all machines incorporate the same encoding strategy. In this way, we have one encoder which is denoted by $\psi$.

\begin{remark}
	The quantized datasets received by the central machine are not distributed according to the normal distribution. Moreover, in general, the tree structure is no longer a property of the new distribution. These facts make the recovery of the structure a rather challenging problem.
\end{remark}

\begin{figure}[t]
	\centering
	\includegraphics[scale=.88]{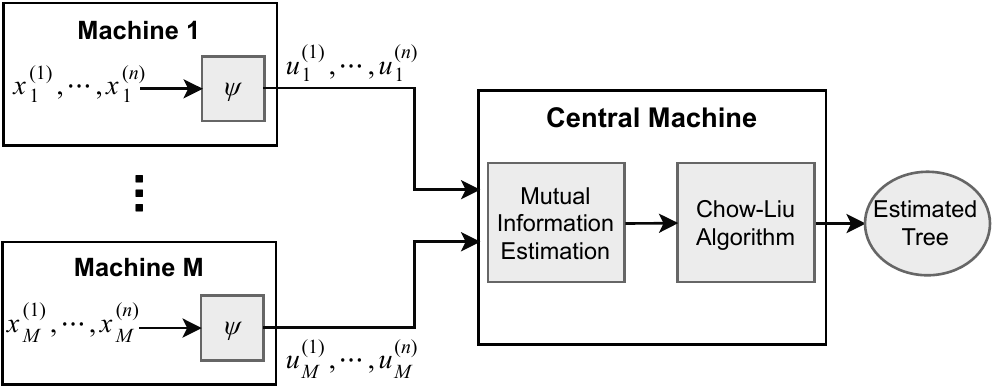}
	\caption{The overall block diagram of our proposed model.}
	\label{fig:whole system}
\end{figure}

In case of having access to the original datasets, the central machine can run the Chow-Liu algorithm in \cite{chow1968approximating} which gives the maximum likelihood (ML) tree structure \cite{tan2010learning}. In the Chow-Liu algorithm, the mutual information between any pair of vertices are estimated and used as the edges weights of a complete graph between vertices. The maximum weight spanning tree (MWST) gives the ML tree. Therefore, the central part of the Chow-Liu algorithm is to estimate the mutual information efficiently.

In graph theory, there are two efficient algorithms for solving the MWST problem: Kruskal \cite{kruskal1956shortest} and Prim \cite{prim1957shortest} algorithms. Throughout this paper we incorporate the Kruskal algorithm for finding the MWST. In Kruskal algorithm, the edges weighs are sorted in descending order and at each step an edge with the highest weight which does not form a cycle is added to the current forest. This algorithm continues until all vertices are covered. The output of this algorithm depends only on the order of edges weights.

In GGMs, the mutual information between any pair of variables, say $x_j$ and $x_k$, is obtained by
\begin{equation} \label{eq:mutual infromation normal}
I(x_j;x_k) = -\frac{1}{2} \ln \left(1-\rho_{jk}^2\right),
\end{equation}
where $\rho_{jk}$ is the correlation coefficient between $x_j$ and $x_k$.  According to \eqref{eq:mutual infromation normal}, one way to estimate the mutual information of two normal variables is to estimate their correlation coefficients first. In our problem setting where each variable is stored in a different machine, estimation of the correlation coefficients is a difficult task. This is due to the fact that  one needs to calculate the statistic $\sum_i x_j^{(i)} x_k^{(i)}$. Computing such a statistic needs transmission of  $x_j$ and $x_k$ samples to the central machine imposing   high communication cost which is prohibitive in band-width limited networks.

\subsection{Proposed Methods} \label{sec:approaches}

Based on our system model, we need to design some efficient ways that machines communicate with the central machine and to implement an algorithm to recover the tree structure at the central machine. We propose two techniques to achieve these goals which are described as follow. 

\subsubsection*{\textbf{Sign Method}}
Each data point is quantized to a single bit by the $\mathrm{sign}$ function, i.e.  
$$
\left(\mathrm{sign}(x_j^{(1)}),\cdots,\mathrm{sign}(x_j^{(n)})\right) = \psi\left(x_j^{(1)},\cdots,x_j^{(n)}\right).
$$
Since the encoder maps each data point to a single reconstruction point independent of the other points, i.e. $u_j^{(i)} = \mathrm{sign}(x_j^{(i)})$, the received data at the central machine is an i.i.d. sequence. However, as it mentioned earlier, the received data does not possess the tree structure of the original data. 

Even though the Chow-Liu algorithm is only applicable for reconstructing tree structures, the central machine uses it to infer ``a tree''  structure embedded in the new distribution. This scheme is called \emph{sign method}.   

Through analysis presented in Section \ref{sec:structure learning using binary} we show that the tree reconstructed by the sign method is the true underlying structure with high probability. Our simulation results also support our analysis.




\subsubsection*{\textbf{Per-symbol Quantization}}
In this scheme, each machine quantizes its data points to $2^R$ possible reconstruction points independent of the other points.   Therefore, the received quantized points are i.i.d. and non-Gaussian.  However, at the central machine, the distribution is assumed to be normal and the correlations are estimated based on the received quantized datasets. 

It worth mentioning that the encoder part of the sign method is a special case of the encoder used here with $R=1$. In contrary, the tree reconstruction algorithms used at the central machine are different. 

We provide an upper bound on the error of the correlation estimation in this case in Section \ref{sec:correlation estimation}. Simulation results also indicate by consuming a few bits for quantization, the estimated structure is often same as the structure obtained by the original data.

\section{Structure Learning with Signs} \label{sec:structure learning using binary}
In the sign method, the machine \machine{j} transmits $u^{(i)}_j = \mathrm{sign}\big(x^{(i)}_j \big)$ for $ i=1, \cdots, n$ to the central machine. Note that since $x_j \sim \mathcal{N}(0, 1)$, $u_j$ is a uniform Bernoulli variable over $\{-1, +1\}$. 

The central machine receives the binary data from all machines to form the quantized dataset $\{\vect u^{(1)}, \cdots, \vect u^{(n)}\}$ where $\vect u^{(i)} \in \{-1,+1\}^d$. Since the dimensions of the original normal vector $\vect x$ are correlated, dimensions of $\vect u$ are dependent as well. Although there is a simple map between the original normal vector and the signs, no closed form probability mass function (pmf) exists for $d \geq 4$. However, some approximations with desirable accuracies are proposed in \cite{bacon1963approximations}. 

Applying  the Chow-Liu algorithm on $\{\vect u^{(1)}, \cdots, \vect u^{(n)}\}$, the central machine obtains an estimate of the underlying structure. The remainder of this section is devoted to analyze the probability of incorrect structure recovery in a non-asymptotic regime. 

\subsection{Order Preserving of Mutual Information}
Essential to the analysis, we show that the order of the true mutual information values between variables remains the same after applying the sign function. In this way, one can claim that by reliable  estimating of the mutual information values of the signs, the true structure can be recovered using the Chow-Liu algorithm.

First, note that if $x_j\sim \mathcal{N}(0,1)$ and $x_k\sim \mathcal{N}(0,1)$ are  jointly normal with the correlation coefficient  $\rho_{jk}$, then the joint pmf of the corresponding signs $u_j$ and $u_k$ can be expressed as \cite{bacon1963approximations}
\begin{equation} \label{eq:binary pmf}
\begin{array}{c|cc}
u_j \backslash u_k & -1 & +1	\\
\hline
-1 & \theta_{jk}/2  & (1-\theta_{jk})/2	\\
+1 & (1-\theta_{jk})/2 & \theta_{jk}/2	
\end{array}
\end{equation}
where 
\begin{equation} \label{eq: theta}
\theta_{jk} = \frac{1}{2} + \frac{\arcsin(\rho_{jk})}{\pi}.
\end{equation}
Therefore, the mutual information between $u_j$ and $u_k$ can be written as  
\begin{equation} \label{eq:mutual information binary}
I(u_j; u_k) = 1 - h(\theta_{jk}),
\end{equation}
where $h(.)$ is the \emph{binary entropy function} given by
\begin{equation} \label{eq:binary entropy function}
h(\theta) = -\theta \log(\theta) - (1 - \theta) \log(1 - \theta).
\end{equation}

	

	

\begin{lemma}\label{lemma:mutual inf binary}
The sign function is an order preserving of the mutual information on Gaussian random variables. 
\end{lemma}
\begin{proof}
Consider two pairs of variables $(x_j, x_k)$ and $(x_r, x_s)$ from a GGM such that $I(x_j; x_k) > I(x_r; x_s)$. Let $(u_j, u_k)$ and $(u_r, u_s)$ be the corresponding sign variables. We need to show that $I(u_j; u_k) > I(u_r, u_s)$.

According to \eqref{eq:mutual infromation normal}, if $I(x_j; x_k) > I(x_r; x_s)$ then $\vert \rho_{jk} \vert > \vert \rho_{rs} \vert$. Since the $\arcsin$ is a monotonic function, we have
\begin{equation} \label{eq:temp1}
\arcsin |\rho_{jk}| > \arcsin |\rho_{rs}|.
\end{equation}
Since $\arcsin$ is odd, using \eqref{eq:temp1} and \eqref{eq: theta} we have
\begin{equation} \label{eq:temp ineq}
\vert \theta_{jk} - \frac{1}{2} \vert > \vert \theta_{rs} - \frac{1}{2} \vert.
\end{equation}
Consider the case where $\theta_{jk} > \frac{1}{2}$ and $\theta_{rs} > \frac{1}{2}$. Then, from \eqref{eq:temp ineq} we have $\theta_{jk} > \theta_{rs} $. Since $h(\theta)$ is a descending function for $1/2 < \theta < 1$, we have $1- h(\theta_{jk}) >1- h( \theta_{rs}) $ which is the desired result.

For the case where  $\theta_{jk} > \frac{1}{2}$ and $\theta_{rs} < \frac{1}{2}$, from \eqref{eq:temp ineq} we have $\theta_{jk} > 1-\theta_{rs}$. Since $h(\theta) = h(1-\theta)$ and it is descending for $1/2 < \theta < 1$, we have $1-h(\theta_{jk}) > 1-h(1-\theta_{rs})$ which is again the desired result. Similar arguments can be applied to the other two cases.

\end{proof}


\subsection{Probability of Incorrect Ordering}

In the Chow-Liu algorithm on the signs, we rely on the estimates of the mutual information between all pairs of variables, $(j,k) \in \mathcal{V}^2$. It is shown in \cite{el-Gamal2017rate} that the following estimator for $\theta_{jk}$ is optimal in the sense that it is unbiased and has minimum variance (UMVE),
\begin{equation}\label{eq:theta estimator}
\hat{\theta}_{jk} = \frac{1}{n} \sum_{i=1}^{n} \mathbb{I}(u^{(i)}_j u^{(i)}_k = 1),
\end{equation}
where $\mathbb{I}(.)$ is the indicator function. By substituting $\hat\theta_{jk}$ into \eqref{eq:mutual information binary}, an estimator for the mutual information of $u_j$ and $u_k$ is obtained which we denote it by $\hat I(u_j; u_k)$.


Let $e = (j, k) \in \mathcal{V}^2$ be a pair of nodes in the network. For simplicity in the notation, we use $\rho_e$, $\theta_e$ and $I_e$ for $\rho_{jk}, \theta_{jk}$ and $I(u_j; u_k)$, respectively.


\begin{definition}[Crossover Event]
	Let $e$ and $e'$ be two pairs of nodes in the graph such that $I_e > I_{e'}$, the crossover event occurs when we estimate the mutual information in the reverse order, i.e. $\hat{I}_e \leq \hat{I}_{e'}$.
\end{definition}

The notion of \emph{crossover event} was previously used by Tan et al. in \cite{tan2010learning}. 
The crossover event indicates a change in the ordering of the values of mutual information estimates. However, it does not necessarily  lead to an incorrect tree recovery. On the other hand, if the estimated tree differs from the original tree $\mathcal{T}$, then at least one crossover event has occurred.

The mutual information in \eqref{eq:mutual information binary} is a monotonic function of $|\theta-1/2|$ (see the proof of Lemma \ref{lemma:mutual inf binary}). Therefore, the probability of crossover event for pairs $e$ and $e'$ with $I_e > I_{e'}$ can be stated as
\begin{align} \label{eq:crossover prob}
\Pr\left(\hat I_{e} \leq \hat I_{e'} \right) = \Pr\left(|\hat\theta_e - \frac{1}{2}| \leq |\hat\theta_{e'} - \frac{1}{2}| \right).
\end{align}

\begin{lemma} \label{lemma:sign flip}
For the probability of the crossover event, it can be assumed $\theta_{jk} > 1/2$ for all $(j,k) \in \mathcal{V}^2$.
\end{lemma}
\begin{proof}
Let $e = (j, k)$ and $e' = (r, s)$ be two arbitrary pairs of variables such that $I_e > I_{e'}$ or equivalently, $|\theta_e - 1/2| > |\theta_{e'}- 1/2|$. The probability of crossover event is 
\begin{equation} \label{eq:proof temp eq1}
\Pr\left(|\hat\theta_e - \frac{1}{2}| \leq |\hat\theta_{e'} - \frac{1}{2}| ~\bigg\vert~ |\theta_e - \frac{1}{2}| > |\theta_{e'} - \frac{1}{2}|\right).
\end{equation}
Let us assume $\theta_e < 1/2$ and $\theta_{e'} > 1/2$ (the other cases can be argued similarly). We define new variable $\tilde u_j = -u_j$. It is clear that the joint pmf of $\tilde{u}_j$ and $u_k$ is given by \eqref{eq:binary pmf} with parameter $\tilde\theta = 1 - \theta_e$.
Thus, $|\theta_e - 1/2| = |\tilde\theta_e - 1/2|$. Similarly,  there exists the following relation between the estimators $\hat\theta_e$ and $\hat{\tilde\theta}$
\begin{align*}
\hat{\tilde\theta} &= \frac{1}{n} \sum_{i=1}^n \mathbb{I}(\tilde u_j^{(i)} u_k^{(i)} = 1) =\frac{1}{n} \sum_{i=1}^n \mathbb{I}( -u_j^{(i)} u_k^{(i)} = 1) \\
&= \frac{1}{n} \sum_{i=1}^n \left(1-\mathbb{I}(u_j^{(i)} u_k^{(i)} = 1) \right) = 1 - \hat\theta_e.
\end{align*}
Thus, $|\hat\theta_e - 1/2| = |\hat{\tilde\theta} - 1/2|$. Therefore, the crossover probability in \eqref{eq:proof temp eq1} can be expressed as
\begin{equation*}
\Pr\left(|\hat{\tilde\theta} - \frac{1}{2}| \leq |\hat\theta_{e'} - \frac{1}{2}| ~\bigg\vert~ |\tilde\theta - \frac{1}{2}| > |\theta_{e'} - \frac{1}{2}|\right),
\end{equation*}
where both $\tilde\theta$ and $\theta_{e'}$ are greater than $1/2$. 
\end{proof}

Lemma \ref{lemma:sign flip} shows that without loss of generality, we can assume all $\theta_{jk}$s are greater than $1/2$ which is equivalent to assuming all correlations are positive (see equation \eqref{eq: theta}). Note that, the condition $\theta_{jk}> \frac{1}{2}$ does not imply that the estimator $\hat\theta_{jk}$ is also greater than $1/2$.

In the following lemma, to make the exposition of the ideas easier, we assume both $\hat\theta_e$ and $\hat\theta_{e'}$ are greater than $1/2$. However, in the supplementary material, we provide an upper bound on \eqref{eq:crossover prob} for all cases.
%

\begin{lemma} \label{lemma: crossover event chernoff}
	Let $\{\vect x^{(1)}, \cdots, \vect x^{(n)}\}$ be $n$ i.i.d. samples drawn from a $d$-dimensional GGM. Assume that the variables have zero means and unit variances. Then, the probability of crossover event for two pairs $e=(j,k)$ and $e'=(r, s)$ with $\theta_e > \theta_{e'}$, is upper bounded by
	\begin{equation}\label{eq:crossover upper bound hoeff}
	\Pr\left(\hat\theta_{e} \leq \hat\theta_ {e'} \right) \leq e^{- n E},
	\end{equation}
	where $E = \ln (p_0 + 2\sqrt{p_1 p_2})$ and 
	\begin{align}
	p_0 &= \Pr\left(u_j u_k = u_r u_s \right), \label{eq:p0}\\
	p_1 &= 	\Pr\left(u_j u_k = -1, u_r u_s = 1\right), \label{eq:p1}\\
	p_2 &= 	\Pr\left(u_j u_k = 1, u_r u_s = -1 \right) \label{eq:p2}.
	\end{align}
Moreover, the exponent $E$ is the tightest possible, i.e.,
\begin{equation}
E = \lim_{n \to \infty} -\frac{1}{n}\ln \Pr\left(\hat\theta_{e} \leq \hat\theta_ {e'} \right).
\end{equation}
\end{lemma}
\begin{proof}
	Consider two pairs of nodes $e=(j,k)$ and $e'=(r,s)$. 
	By defining a random variable $T_i$ as
	\begin{equation} \label{eq:T_i}
	T_i = \mathbb{I}(u_r^{(i)} u_s^{(i)} = 1) -\mathbb{I}(u_j^{(i)} u_k^{(i)} = 1),
	\end{equation} 
	the probability of crossover event can be written as
	\begin{align*}
	\Pr\left(\hat\theta_{e} \leq \hat\theta_{e'} \right) &= \Pr\left(\sum_{i=1}^n T_i \geq 0\right)	\\ 
	&= \Pr\left(e^{\lambda\sum_{i=1}^n T_i} \geq 1\right), \qquad \lambda > 0	\\
	&\stackrel{(a)}{\leq} \mathbb{E}\left[e^{\lambda\sum_{i=1}^n T_i}\right]	\\
	&= \left(\mathbb{E}\left[e^{\lambda T}\right] \right)^n	\\
	&= \left(p_0 + p_1 e^{\lambda} + p_2 e^{-\lambda} \right)^n,
	\end{align*}
	where the inequality $(a)$ is by Markov's inequality. The random variable $T$ can take values $[0, 1, -1]$ with probabilities $[p_0, p_1, p_2]$ defined in \eqref{eq:p0}-\eqref{eq:p2}.  
	By minimizing the last expression for $\lambda > 0$, we obtain the Chernoff bound as follows
	\begin{equation}
	\Pr\left(\hat\theta_{e} \leq \hat\theta_{e'} \right) \leq \left(p_0+2\sqrt{p_1 p_2} \right)^n = e^{-n E},
	\end{equation}
	where $E = \ln(p_0 + 2\sqrt{p_1 p_2})$. Incorporating Theorem 2.1.24 in \cite{dembo2010large}, the exponent $E$ is indeed tight.
\end{proof}

Unfortunately, there is no closed-form solution for the probabilities in \eqref{eq:p0}-\eqref{eq:p2}. However, when $e$ and $e'$ share a common variable, these probabilities can be obtained analytically. For example, if $u_k = u_s$ (i.e. $u_k$ is the common variable between $e$ and $e'$), then the probabilities are given by \cite{bacon1963approximations}
\begin{align}
p_0 &= \frac{1}{2} + \frac{\arcsin\rho_{jk}\rho_{ks}}{\pi},	\\
p_1 &= \frac{1}{4} + \frac{-\arcsin\rho_{jk} + \arcsin\rho_{ks} -\arcsin\rho_{jk}\rho_{ks}}{2\pi}, \\
p_2 &= \frac{1}{4} + \frac{\arcsin\rho_{jk} - \arcsin\rho_{ks} -\arcsin\rho_{jk}\rho_{ks}}{2\pi}.
\end{align}
In particular, for the star structure, where all the true edges share a common node, the bound of Lemma \ref{lemma: crossover event chernoff} can be calculated in a closed-form.

In the following lemma, we propose another upper bound on the probability of the crossover event using Hoeffding's bound. This bound is not tight, but it yields a closed form expression which can be used to obtain a closed form bound on the probability of incorrect tree estimation. 
\begin{lemma} \label{lemma: crossover event hoeffding}
The probability of the crossover event for two pairs $e$ and $e'$ with $\theta_{e} > \theta_{e'}$, is upper bounded by
\begin{equation}
\Pr\left(\hat\theta_e \leq \hat\theta_{e'} \right) \leq e^{-\frac{1}{2} n \Delta\theta_{e,e'}^2},
\end{equation}
where $\Delta\theta_{e,e'} = \theta_{e} - \theta_{e'}$.
\end{lemma}
\begin{proof}
	Since $\hat{\theta}_e$ and $\hat{\theta}_{e'}$ are unbiased estimators for $\theta_e$ and $\theta_{e'}$, we have
	\begin{align*}
	\Pr\left(\hat\theta_{e} \leq \hat\theta_ {e'} \right) &= \Pr\left(\hat\theta_{e'} - \hat\theta_ {e} - \mathbb{E}\left[\hat{\theta}_{e'} - \hat{\theta}_e\right] \geq \Delta\theta_{e,e'} \right).	
	\end{align*}
	Defining variable $T_i$ as \eqref{eq:T_i}, we can write the above probability as
	\begin{equation*}
	\Pr\left(\hat\theta_{e} \leq \hat\theta_ {e'} \right) = \Pr\left(\frac{1}{n}\sum_{i=1}^n \left( T_i - \mathbb{E}\left[T_i\right] \right) \geq \Delta\theta_{e,e'} \right).
	\end{equation*}
	It is clear that $T_i \in \{-1, 0, 1\}$, thus it is bounded in interval $[-1, 1]$. Using the Hoeffding's inequality we can obtain an upper bound on the probability of crossover event as
	\begin{equation*}
	\Pr\left(\hat\theta_{e} \leq \hat\theta_ {e'} \right) \leq  e^{-\frac{1}{2} n \Delta\theta_{e,e'}^2}.
	\end{equation*}
\end{proof}

\subsection{Probability of Incorrect Recovery}
In this section, we are interested in bounding $\Pr\left(\widehat{\mathcal{T}}\neq \mathcal{T} \right)$ where  $\widehat{\mathcal{T}} = (\mathcal{V}, \widehat{\mathcal{E}})$ refers to the estimated tree from our proposed sign method. If we assume that  any error in the ordering of mutual information estimates may lead to incorrect recovery of the tree structure, then by the union bound we have 

%
\begin{equation} \label{eq:tree error union bound}
\Pr\left(\widehat{\mathcal{T}} \neq \mathcal{T} \right) \leq \sum_{e,e' \in \mathcal{V}^2} e^{- \frac{1}{2} n \Delta\theta_{e,e'}^2}
\end{equation}

In the following theorem, we improve on the preceding bound by removing some of the crossover events. Moreover, we obtain a more compact and suitable formula for the bound.

\begin{theorem} \label{thm:tree error bound}
	Let   $\{\vect x^{(1)}, \cdots, \vect x^{(n)}\}$ be $n$ i.i.d. samples drawn from a $d$-dimensional tree-structured GGM. We construct $n$ binary vectors $\{\vect u^{(1)}, \cdots, \vect u^{(n)}\}$ where $\vect u^{(i)} \in \{-1,+1\}^d$ is the sign vector of $\vect x^{(i)}$. Assume for all $(j,k)\in\mathcal{E}$, $\alpha \leq \rho_{jk} \leq \beta$ where $0 < \alpha < \beta < 1$. Then the probability of incorrect tree recovering using the Chow-Liu algorithm via the binary vectors is upper bounded by
	\begin{equation} \label{eq:tree err bound binary data}
	\Pr\left(\widehat{\mathcal{T}} \neq \mathcal{T}\right) \leq d^3 e^{- \frac{1}{2} n h^2(\alpha, \beta)},
	\end{equation}
	where $h(\alpha, \beta) = \dfrac{1}{\pi}\left(\arcsin \alpha - \arcsin \alpha\beta \right)$.
\end{theorem}

The rest of this section is devoted to the proof of the theorem.
By removing any edge of the tree, $\mathcal{T}$ splits into two separate sub-trees. Let us assume that in the procedure of tree reconstruction, a true edge $e=(j, k)$ does not appear in $\widehat{\mathcal{T}}$. Removing $e$ from the tree splits  $\mathcal{T}$ into  $\mathcal{T}_1 = (\mathcal{V}_1, \mathcal{E}_1)$ and $\mathcal{T}_2 = (\mathcal{V}_2, \mathcal{E}_2)$. Hence, there should be  an edge $e'=(r, s)\in\widehat{\mathcal{E}}$ which is not in $\mathcal{E}$ and connects $\mathcal{T}_1$ and $\mathcal{T}_2$. This is due the fact that we constrain the tree to be connected. 
%
%
%
The following lemma shows that $e$ is in fact the strongest edge connecting $\mathcal{T}_1$ and $\mathcal{T}_2$. 

\begin{lemma} \label{lemma:corr decay binary}
	Consider a tree-structured GGM $\mathcal{T} = (\mathcal{V}, \mathcal{E})$. Let $e\in \mathcal{E}$ be an edge which connects two sub-trees $\mathcal{T}_1 = (\mathcal{V}_1, \mathcal{E}_1)$ and $\mathcal{T}_2 = (\mathcal{V}_2,\mathcal{E}_2)$. Then for any node pair $(r, s)$ where $r \in \mathcal{V}_1$ and $s \in \mathcal{V}_2$, we have
	\begin{equation*}
	\theta_{e} \geq \theta_{rs}.
	\end{equation*}
\end{lemma} 
\begin{proof}
	In any tree-structured Gaussian distributions, we have \cite{cover2006elements}
	\begin{equation} \label{eq:corr decay}
	\rho_{rs} = \prod_{e \in \mathrm{Path}(r,s)}  \rho_e,
	\end{equation}
	where $\mathrm{Path}(r,s)$ is the set of edges in the path connecting $r$ and $s$ in the tree. This means that the correlation of $(r,s)$ is less than any edge in the path connecting them. Since $r \in \mathcal{V}_1$ and $s \in \mathcal{V}_2$, then the path connecting $r$ and $s$ must include the edge $e$. On the other hand, according to \eqref{eq: theta}, if $\rho_{rs} \leq \rho_e$ then $\theta_{rs} \leq \theta_{e}$ since the function $\arcsin$ is monotonic.
\end{proof}

According to the Kruskal algorithm, the estimated $\widehat{\theta}_{jk}$s for all $(j,k)\in \mathcal{V}^2$ are sorted in a descending order. Scanning from the top of the list, an edge is selected as a part of the tree if it does not create any cycle with the previous picked edges. In the case of error, a true edge like $e \in \mathcal{E}$ is replaced by another false edge $e' \notin \mathcal{E}$ if $\widehat\theta_{e'} > \widehat\theta_{e}$ and $e'$ connects the two subtrees created by removing $e$. On the other hand, from Lemma \ref{lemma:mutual inf binary}, we know that $\theta_{e'} < \theta_{e}$. Thus, replacing $e$ by $e'$ implies a crossover event on them.

Let $e\in \mathcal{E}$ be an edge which connects two sub-trees $\mathcal{T}_1 = (\mathcal{V}_1,\mathcal{E}_1)$ and $\mathcal{T}_2 = (\mathcal{V}_2,\mathcal{E}_2)$. Let
$$\mathcal{C}(e) = \{e' \mid e' \notin \mathcal{E} \text{ and connects } \mathcal{T}_1 \text{ and } \mathcal{T}_2 \},$$
be the set of all candidate false edges which can substitute the edge $e$ in the estimated tree. Then, from the union bound and incorporating Lemma \ref{lemma: crossover event hoeffding} and Lemma \ref{lemma:corr decay binary}, we have
\begin{align}
\Pr\left(\mathcal{T} \neq \widehat{\mathcal{T}}\right) &\leq \sum_{e \in \mathcal{E}} \Pr\left(e \notin \widehat{\mathcal{E}}\right)	\nonumber \\
&\leq \sum_{e \in \mathcal{E}} \sum_{e' \in \mathcal{C}(e)} \Pr\left(\hat{\theta}_{e'} \geq \hat{\theta}_e \right)	\nonumber \\
&\leq \sum_{e \in \mathcal{E}} \sum_{e' \in \mathcal{C}(e)} e^{-\frac{1}{2} n \Delta\theta_{e,e'}^2}.	\label{eq:err prob upper bound mid}
\end{align}

If we obtain a lower bound on $\Delta\theta_{e,e'}$ which is independent of the tree structure, say $\Delta_0 \leq \Delta\theta_{e,e'}$, then we have 
\begin{align}\label{eq:final_rep}
\Pr\left(\mathcal{T} \neq \widehat{\mathcal{T}}\right) &\leq 
d^3 e^{-\frac{1}{2} n \Delta_0^2}.
\end{align}
To this end, we need the following definition. 

\begin{definition} [Strongest Rival]\label{def:stronges neighbor}
	The strongest rival of an edge $e$ is an edge  $e^* \in \mathcal{C}(e)$ with the highest $\theta_{e^*}$.
\end{definition}

To find the strongest rival for an edge $e=(j,k)$, we use the correlation decay property of the tree-structured GGMs. In fact, it is easy to see that  $e^* $ is either an edge connecting node $j$ to one of the neighbors of node $k$ or vice versa. In other words, the strongest rival of an edge $e$, is one of the its neighbor edges (two edges are neighbor if share a common node). \figurename{}~\ref{fig:most prob edge} illustrates the strongest rival of the edge $e = (j,k)$ in a sample tree. The following lemma gives a lower bound on $\Delta\theta_{e,e^*}$.

\begin{lemma}
	Let $e^*$ be the strongest rival of an edge $e \in \mathcal{E}$. If the correlation coefficients $\rho_{jk}$ for all $(j,k) \in \mathcal{E}$ satisfy $\alpha \leq \rho_{jk} \leq \beta$ where $ 0 < \alpha < \beta < 1$, then $\Delta\theta_{e,e^*} \geq h(\alpha, \beta)$ where
	\begin{equation} \label{eq:h(a,b)}
	h(\alpha, \beta) = \frac{1}{\pi} (\arcsin\alpha - \arcsin\alpha\beta).
	\end{equation}
\end{lemma}
\begin{proof}
	Let $e = (j, k)$. Since $e^* \notin \mathcal{E}$ and it connects a neighbor of $j$ to $k$ (or vice versa), using \eqref{eq:corr decay} we have
	\begin{align*}
	\alpha\rho_{e} \leq \rho_{e^*} \leq \beta\rho_{e}.
	\end{align*}
	Therefore, in order to obtain a lower bound on $\Delta\theta_{e,e^*}$, we need to solve the following constrained optimization problem:
	\begin{align} \label{eq:delta phi opt problem}
	&\min_{\rho_{ e}, \rho_{ e^*}}  \arcsin \rho_{ e} - \arcsin \rho_{e^*},	\\
	&\quad \text{subject to:}	\nonumber	\\
	&\quad\qquad	\alpha\rho_{ e} \leq \rho_{e^*} \leq \beta\rho_{e}, \nonumber	\\
	&\quad\qquad \alpha \leq \rho_{e} \leq \beta.	\nonumber
	\end{align}
	Defining a new parameter $\eta = \rho_{e^*}/\rho_e$, the above optimization problem can be written as
	\begin{align} \label{eq:delta phi opt problem2}
	&\min_{\rho_{e}, \eta} ~\arcsin \rho_{e} - \arcsin \eta\rho_{e},	\\
	&\quad \text{subject to:}	\nonumber	\\
	&\quad\qquad \alpha \leq \rho_{e} \leq \beta \nonumber \\
	&\quad\qquad \alpha \leq \eta \leq \beta. \nonumber
	\end{align}
	Taking derivative with respect to $\rho_e$ and $\eta$ we have,
	\begin{align*}
	\frac{\partial}{\partial\rho_e} &= \frac{1}{\sqrt{1-\rho_e^2}} - \frac{\eta}{\sqrt{1-\eta^2\rho_e^2}} > 0,	\\
	\frac{\partial}{\partial\eta} &= \frac{-\rho_e}{\sqrt{1-\eta^2\rho_e^2}} < 0.
	\end{align*}
	Hence, the minimum is attained at $(\rho_e, \eta) = (\alpha, \beta)$. Thus, $\rho_{ e^*} = \eta\rho_e = \alpha\beta$. By substituting $(\rho_e, \rho_{ e^*}) = (\alpha, \alpha\beta)$ into $\Delta\theta_{e,e^*}$ the lower bound in \eqref{eq:h(a,b)} is obtained.
\end{proof}

With the above lemma, we obtain $\Delta_0 =	h(\alpha, \beta)$ in \eqref{eq:final_rep}. Hence, we complete the proof of  Theorem \ref{thm:tree error bound}.

\begin{remark}
	Since we have not imposed any constraint  on the tree structure, the upper bound in Theorem \ref{thm:tree error bound} is obtained for the worst case of the structure resulting the prefactor of $d^3$ which is tight for the chain structure. However, the prefactor can be reduced in many cases such as the star structure which requires  $d^2$ as the prefactor. However, the prefactor has negligible effect if sufficiently large sample size $n$ is available. 
\end{remark}

\begin{remark} \label{remark:tightness of star structure}
	Replacing $\Delta_0$ for all the rival edges is not optimal in general. However, in the special case of the star structure with equal edge weights, for instance, it is tight. Hence, some prior knowledge about the tree structure  can lead to possibly better bounds.
\end{remark}

\begin{figure}[t]
	\centering
	\includegraphics[scale=1]{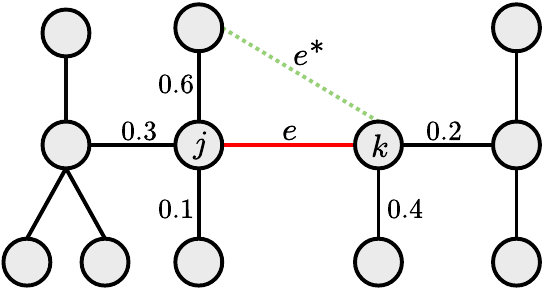}
	\caption{Illustration of the strongest rival edge for $e = (j,k)$ in a sample tree. The weights are the correlation coefficients.}
	\label{fig:most prob edge}
\end{figure}


\section{Structure Learning with Direct Correlation Estimation} \label{sec:correlation estimation}



The Chow-Liu algorithm requires a good estimate of the mutual information between pairs of variables. In GGMs, according to \eqref{eq:mutual infromation normal}, quality of the mutual information estimation depends on the accuracy of correlation estimation. In the sign method (Section \ref{sec:structure learning using binary}), we do not estimate the correlations between normal variables based on the binary data. By contrast, in this section, we aim to quantize data with $R$ bits and  to estimate the correlation coefficients between the normal variables using the quantized data.

Instead of estimating the mutual information in  \eqref{eq:mutual infromation normal} using the quantized data, we use an unbiased estimator for $\rho^2$. This is due to the fact that obtaining an unbiased estimator for the mutual information is a hard problem. The following estimator for $\rho^2$  is unbiased:
\begin{equation}\label{eq:rho^2 hat}
\widehat{\rho^2} = \frac{n}{n+1} \left(\bar\rho^2 - \frac{1}{n} \right), 
\end{equation}
where $\bar \rho$ is the sample correlation coefficient, i.e.
\begin{equation}
\bar \rho = \frac{1}{n} \sum_{i=1}^{n} x^{(i)}_j x^{(i)}_k.
\end{equation}

In our problem setting, we calculate $\bar \rho$ using the $R$-bit quantized data. We define
\begin{equation}\label{eq:rho estimate quantize}
\bar \rho_q = \frac{1}{n} \sum_{i=1}^{n} u_j^{(i)} u_k^{(i)},
\end{equation}
which is the sample correlation coefficient of the quantized data.
Obviously, by increasing the bit rate $R$, $\bar\rho_q$ approaches $\bar{\rho}$.
To measure the performance of the quantization method, we define the relative error as
\begin{equation} \label{eq:samp err func}
\mathrm{err_{rel}} \triangleq \mathbb{E} \left[\left\vert \bar\rho - \bar{\rho}_q \right\vert \right],
\end{equation}
where the expectation is taken over all original and quantized variables $(x_j, x_k, u_j, u_k)$. We call the error function in \eqref{eq:samp err func} as \emph{relative correlation error}. It measures the expected difference between sample correlation coefficient on the original and quantized data. Our main objective is to obtain a good estimate for the true correlation coefficient. Thus, we define the estimation error of a quantization method as
\begin{equation} \label{eq:estimation error}
\mathrm{err_{est}} = \mathbb{E}\left[\left\vert \rho - \bar{\rho}_q \right\vert\right],
\end{equation}
where the expectation is taken over the quantized variables $(u_j, u_k)$. 

Although we aim at designing a quantization method which have a small estimation error, most of existing efficient source coding schemes are designed to minimize the \emph{reconstruction error}. More precisely, denoting the original and quantized versions by $x$ and $u$. These methods quantize the data with the following constraint
\begin{equation} \label{eq:reconst error}
\mathbb{E}\left[(x - u)^2\right] \leq D,
\end{equation}
where $D$ is the maximum allowable reconstruction error. For example, there exists an optimal coding scheme for normal variables based on the rate-distortion theory \cite{cover2006elements}. In our problem setting, minimizing the reconstruction error is not the main objective. Our goal is to minimize the error functions defined in \eqref{eq:samp err func} and \eqref{eq:estimation error}. However, the following theorem guarantees that upper bounding the reconstruction error results in upper bounding the relative correlation error.

\begin{theorem} \label{thm:distortion}
	Let $ \{x_1, \cdots, x_n\}$ and $\{y_1, \cdots, y_n\}$ be two sets of i.i.d. samples from any joint distribution $P(x,y)$ such that $x$ and $y$ have zero means and unit variances. If $ \{u_1, \cdots, u_n\}$ and $ \{v_1, \cdots, v_n\}$ are the corresponding encoded sequences by any quantizer with
	\begin{align*}
	\mathbb{E} \left[ \frac{1}{n} \sum_{i=1}^{n} \left(x_i -  u_i\right)^2   \right] &\leq D_1,	\\
	\mathbb{E} \left[ \frac{1}{n} \sum_{i=1}^{n} \left(y_i -  v_i\right)^2   \right] &\leq D_2,
	\end{align*}
then
	\begin{align} \label{eq:relative err bound}
	\mathrm{err_{rel}}  &\leq 
	\sqrt{D_1} + \sqrt{D_2} + \sqrt{D_1 D_2}.
	\end{align}
\end{theorem}
The proof of the theorem is presented in Appendix \ref{appendix:proof of theorem 2}. The above theorem suggests that designing a coding scheme with small reconstruction error yields a desirable relative correlation error. 

\begin{lemma} \label{lemma:estimation err}
	The estimation error in \eqref{eq:estimation error} can be upper bounded as follow
	\begin{equation} \label{eq:estimation error bound}
		\mathrm{err_{est}} \leq \sqrt{\frac{1+\rho^2}{n}} + \mathrm{err_{rel}}.
	\end{equation}
\end{lemma}
\begin{proof}
	Using the triangle inequality, we have
	\begin{align} \label{eq:eq1}
	\mathbb{E}\left[\vert\rho - \bar\rho_q\vert\right] &\leq \mathbb{E}\left[\vert\rho-\bar\rho\vert\right] + \mathbb{E}\left[\vert\bar\rho-\bar\rho_q\vert\right].
	\end{align}
	On the other hand, using the Jensen's inequality and the convexity of the square function, we have
	\begin{equation} \label{eq:eq2}
	\mathbb{E}^2\left[\vert\rho-\bar\rho\vert\right] \leq \mathbb{E}\left[(\rho-\bar\rho)^2\right] = \frac{1+\rho^2}{n}.
	\end{equation}
	By combining \eqref{eq:eq1} and \eqref{eq:eq2}, the bound in \eqref{eq:estimation error bound} is obtained.
\end{proof}


Next, we obtain upper bounds on $ \mathrm{err_{est}}$ and $\mathrm{err_{rel}}$ for our per-symbol quantization method. 
 In this method, we first quantize each sample to a discrete random variable $u \in \mathcal{U}$ where $\vert \mathcal{U} \vert = 2^R$. Thus, each sample $u$ can be encoded by $R$ bits. Assume we have a standard normal random variable $x \sim \mathcal{N}(0, 1)$ which we want to quantize by $R$ bits. To this end, we construct $2^R$ equally probable bins over the real axis and use bins' centroids as the reconstruction points. These centroids are indeed the members of $\mathcal{U}$ which is used for compression of $x$. To compress $x$, we send index of the bin that $x$ belongs to it by $R$ bits.

For the standard normal distribution, $u$ can have a value from the codebook $\mathcal{U} = \{c_1, \cdots, c_{2^R}\}$ where $c_i$ is the centroid of $i$-th bin. Denoting the $i$-th bin interval by $(a_i, a_{i+1})$, its centroid $c_i$ is obtained by
\begin{equation}
c_i = \frac{2^R}{\sqrt{2 \pi}} \left(e^{-a_i^2/2} - e^{a_{i+1}^2/2}\right), \qquad i = 1, \cdots, 2^R.
\end{equation}
Interval boundaries $\{a_i\}$ are obtained as follow. We first set $a_1 = -\infty$. Then, given $a_{i-1}$, we iteratively obtain $a_i$ as the solution of the following equation: 
\begin{align*}
\int_{a_{i-1}}^{a_i} \mathcal{N}(x; 0, 1) dx = 2^{-R}; \quad i=2, \cdots, 2^R+1.
\end{align*}

The expectation of reconstruction error for this coding scheme is obtained as follows,
\begin{align}
\mathbb{E}\left[\left(x - u\right)^2\right] &=  1 + \sigma_u^2 - 2\sum_{i=1}^{2^R} \int_{-\infty}^{+\infty} c_i x p(x) p(u|x) dx	\nonumber \\
&= 1 + \sigma_u^2 - 2\sum_{i=1}^{2^R} \int_{a_i}^{a_{i+1}} c_i x p(x) dx	\nonumber \\
&= 1 + \sigma_u^2 - 2 \cdot 2^{-R} \sum_{i=1}^{2^R} c_i^2	\nonumber \\
&= 1-\sigma_u^2 \label{eq:reconst err per-sym},
\end{align}
where $\sigma_u^2$ is the variance of discrete variable $u$. Evidently, by increasing the bit rate $R$, $\sigma_u^2$ approaches $1$. 
After encoding the normal variables using the above method, the central machine estimates the correlation between $x_j$ and $x_k$ for any $(j,k) \in \mathcal{V}^2$ using \eqref{eq:rho estimate quantize}. Then, it estimates $\rho^2$ by substituting $\bar\rho_q$ in \eqref{eq:rho^2 hat}. Finally, by applying the Chow-Liu algorithm, the structure of underlying GGM is estimated.

According to Theorem \ref{thm:distortion} and the reconstruction error in \eqref{eq:reconst err per-sym}, the relative correlation estimation error in \eqref{eq:samp err func} is upper bounded by
\begin{equation}
\mathrm{err_{rel}} \leq 2 \sqrt{1-\sigma_u^2} + 1 - \sigma_u^2.
\end{equation}
The variance $\sigma_u^2$ is a function of the bit rate $R$, but it does not have an explicit closed form expression.
By incorporating Lemma \ref{lemma:estimation err}, the estimation error of per-symbol quantizer is upper bounded by
\begin{equation} \label{eq:estimator err per symbol}
\mathrm{err_{est}} \leq 2\sqrt{1-\sigma_u^2} + 1 - \sigma_u^2 + \sqrt{\frac{1+\rho^2}{n}}.
\end{equation}

In Section \ref{sec:experiments}, we will evaluate the performance of this method for structure estimation on real and synthetic datasets.

\section{Experiments} \label{sec:experiments}
In this section, we evaluate the performance of our proposed structure learning methods empirically on some synthetic and real-world datasets. The results show that quantizing samples to few bits is enough to estimate the underlying structure of the model. In all the experiments, double-precision floating-points ($64$ bit) are used for the original data.

\subsection{Synthetic Data}
Synthetic data are generated from a random tree with $d$ nodes. Then, a random weight is assigned to each edge of the tree which corresponds to the correlation coefficient between endpoint variables  of the edge. The correlation coefficient between any pair of variables which are not neighbor, can be obtained by \eqref{eq:corr decay}. Thus, using this weighted tree, the covariance matrix of the GGM is obtained and $n$ i.i.d. samples are drawn from the underlying normal distribution. Finally, dimensions of the data are distributed among $d$ machines.

In \figurename{}~\ref{fig:tree err vs n and R}, the performance of the sign method and per-symbol quantization for different values of $n$ and $R$ is plotted. In this experiment, the underlying GGM has $20$ variables. To approximate the probability of error, for each sample size $n$, we run the algorithms $1000$ times and count the number of incorrect estimated trees. The experiment shows that recovering the structure from the sign method yields better performance than the per-symbol method for $R = 1$. Interestingly,  the error of $4$-bit per-symbol is very close to the non-quantized (original data) curve. This means that the accuracy of correlation estimation using $4$-bit quantized data is sufficient to achieve the centralized performance in structure estimation.

%

\begin{figure}[t]
	\centering
	\includegraphics[scale=.5]{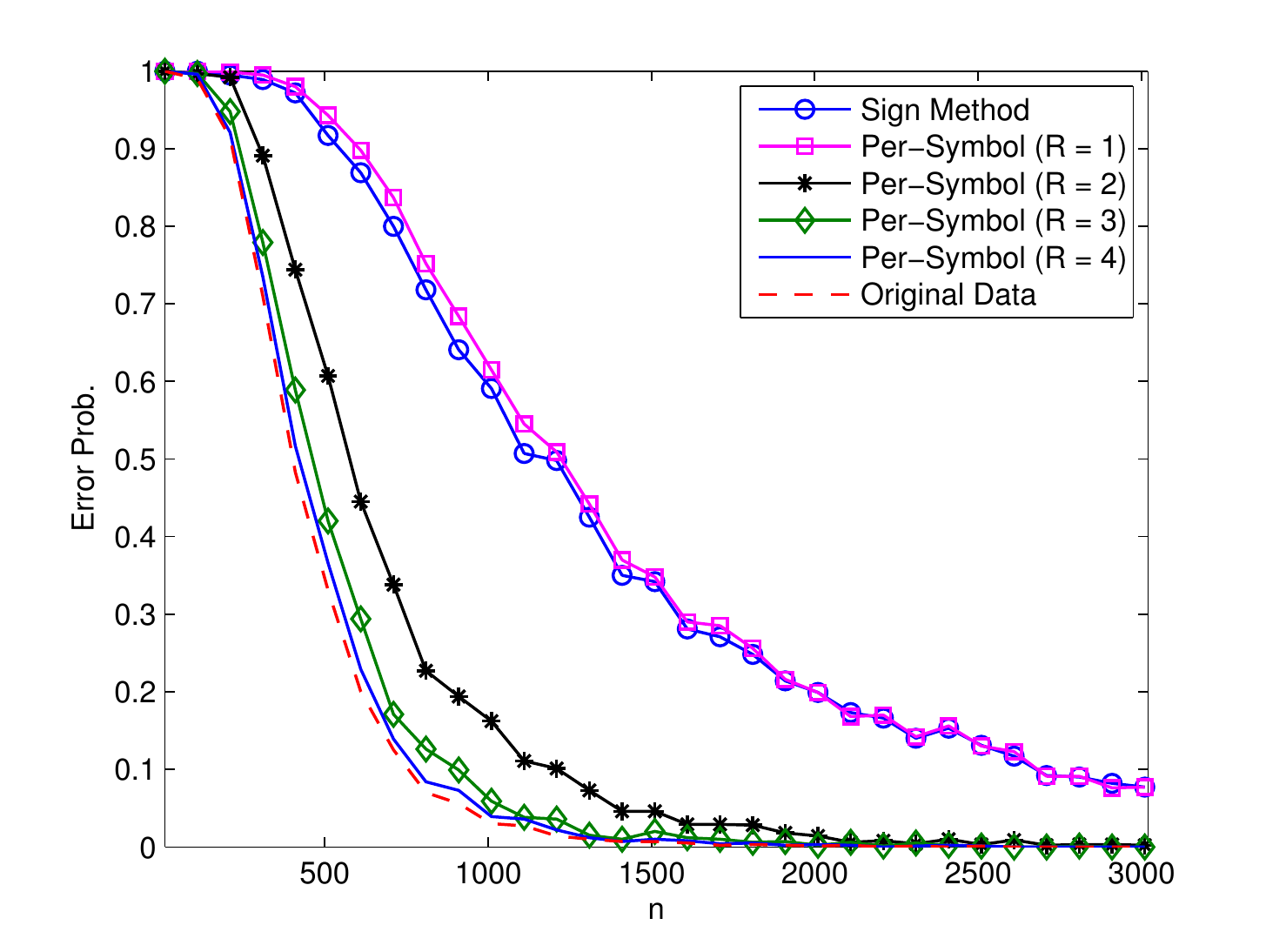}
	\caption{The structure estimation error for different values of $R$ and $n$. Here, the random GGM has $20$ nodes.}
	\label{fig:tree err vs n and R}	
\end{figure}



\subsubsection{Evaluation of the Error Bounds}
We first focus on a simple tree with three nodes and correlations coefficients $\rho_1 = 0.9$ and $ \rho_2=0.1,$ as depicted in \figurename{}~\ref{fig:crossover}. A crossover event happens if the estimate of the mutual information associated to $e'$ exceeds that of $e$. To evaluate the error bounds of this event obtained in Lemma \ref{lemma: crossover event chernoff} and Lemma \ref{lemma: crossover event hoeffding}, we also provide the exact error which can be calculated by a brute force summation of the tail probability of $\Pr(\hat\theta_e \leq \hat \theta_{e'})$. \figurename{}~\ref{fig:crossover-err} depicts the probability of crossover event versus the number of samples. As can be seen, the upper bound of Lemma \ref{lemma: crossover event chernoff} converges to the exact error faster than the bound of Lemma \ref{lemma: crossover event hoeffding}. 

In \figurename{}~\ref{fig:crossover-err-exponent}, the exponent of exact error, Chernoff bound (Lemma \ref{lemma: crossover event chernoff}) and Hoeffding bound (Lemma \ref{lemma: crossover event hoeffding}) for the structure of \figurename{}~\ref{fig:crossover}  are compared. In the figure, the quantity $-\frac{1}{n} \ln\Pr(\hat\theta_{e} \leq \hat\theta_{e'})$ is plotted for various sample sizes.  As stated in Lemma \ref{lemma: crossover event chernoff}, the bound obtained based on Chernoff bound is tight in the exponent. However, Hoeffding bound is not tight in this case.

\begin{figure}
	\centering
	\includegraphics[scale=0.8]{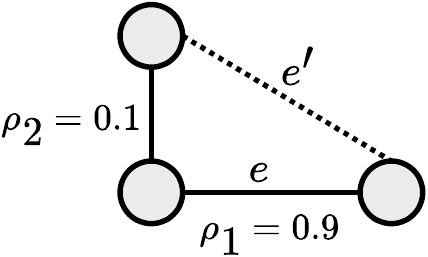}
	\caption{A sample sub-tree of three nodes for evaluation of the crossover error bound.}
	\label{fig:crossover}
\end{figure}

\begin{figure}
	\centering
	\includegraphics[width=\linewidth]{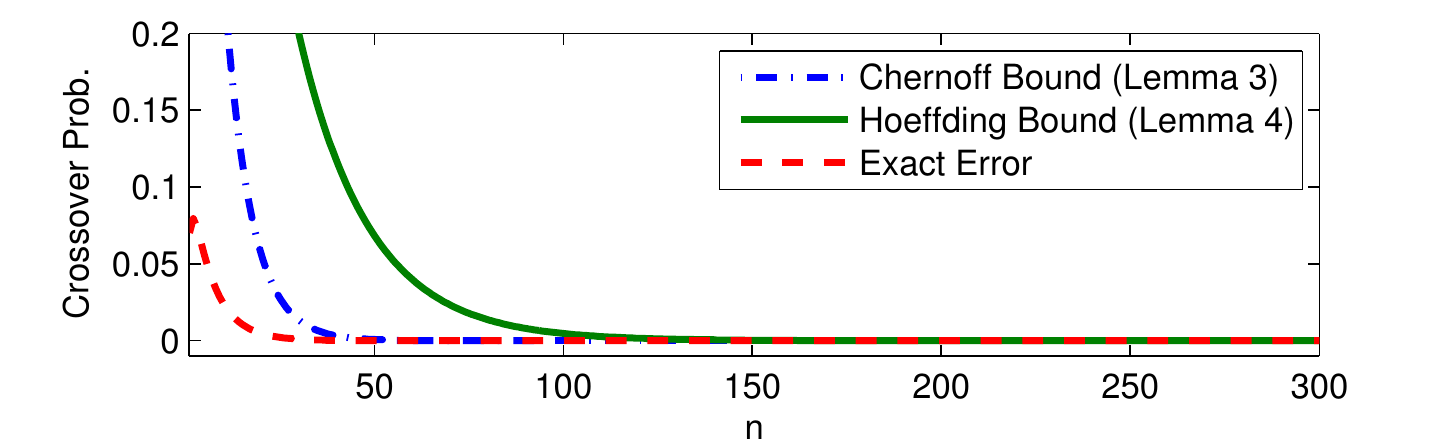}
	\caption{Probability of the crossover event for $e$ and $e'$ in \figurename{}~\ref{fig:crossover}.}
	\label{fig:crossover-err}
\end{figure}

\begin{figure}
	\centering
	\includegraphics[width=\linewidth]{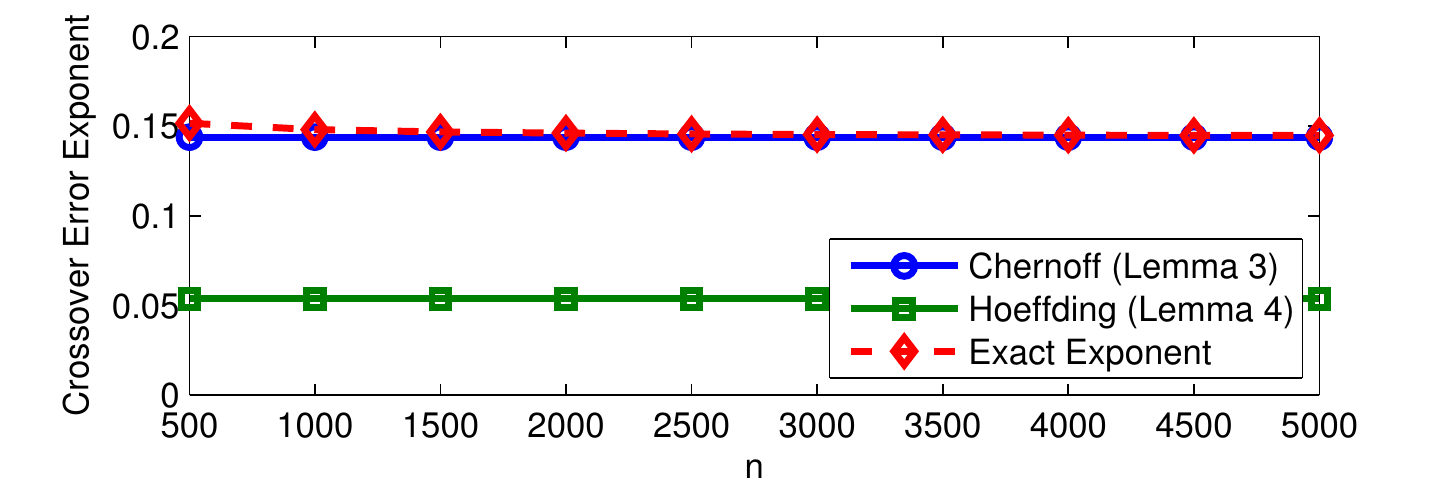}
	\caption{The crossover error exponent for $e$ and $e'$ in \figurename{}~\ref{fig:crossover}.}
	\label{fig:crossover-err-exponent}
\end{figure}

\figurename{}~\ref{fig:tree-err-bnd} shows the probability of incorrect recovery of the tree structure using the sign method as a function of sample size. We have used  a star structured tree with $20$ nodes and correlations of $0.5$ which due to Remark \ref{remark:tightness of star structure} is the worst structure. 

In \figurename{} \ref{fig:relative-err-exponent}, the exponent of the bound of Theorem \ref{thm:distortion} on $\mathrm{err_{rel}}$ is plotted as a function of bit rate $R$. In this experiment, the true correlation coefficient is $\rho=0.5$ and the sample size is $n=1000$. The empirical error curve is obtained by averaging over 1000 runs. In the figure, the $y$ axis represents the quantity $-\frac{1}{R} \ln (\mathrm{err_{rel}})$. As can be seen from the figure, the upper bound is not tight in the exponent for Gaussian data. Note that the error bound in Theorem \ref{thm:distortion} is valid for any distribution and any quantization method.




\begin{figure}
	\centering
	\includegraphics[width=\linewidth]{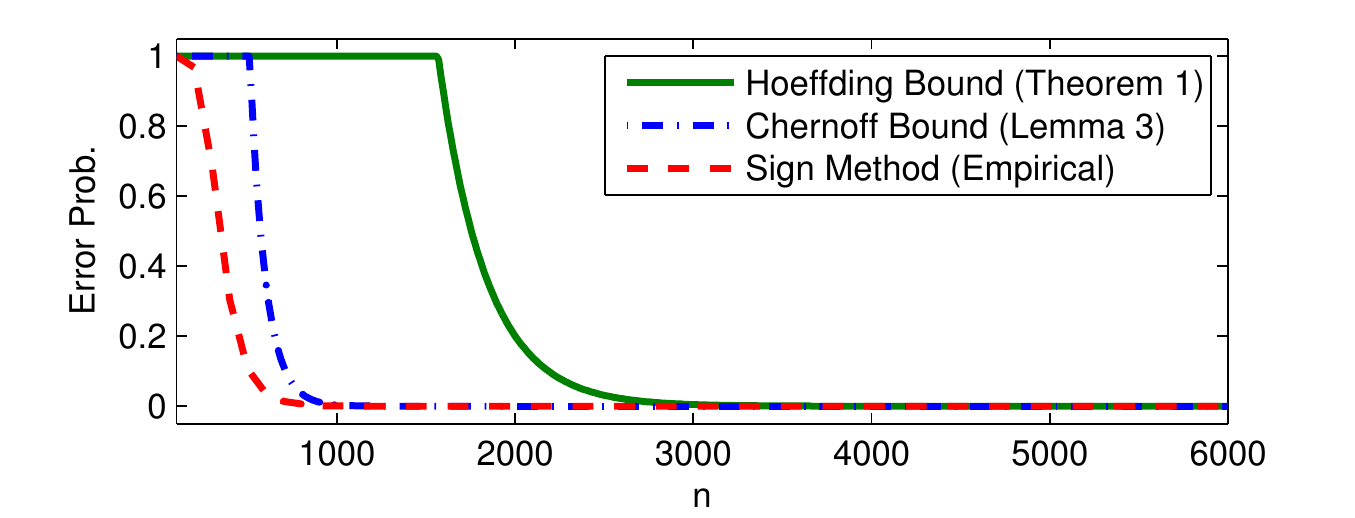}
	\caption{Probability of incorrect tree estimation for the star structure.}
	\label{fig:tree-err-bnd}
\end{figure}

\begin{figure}
	\centering
	\includegraphics[width=\linewidth]{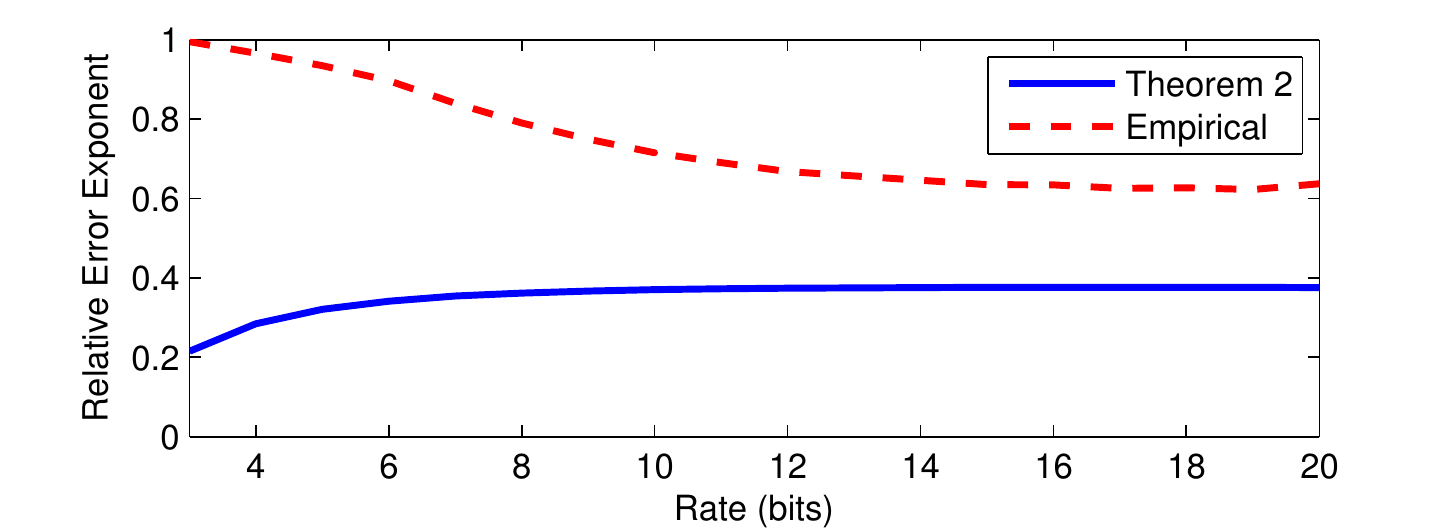}
	\caption{Relative error exponent of per-symbol method.}
	\label{fig:relative-err-exponent}
\end{figure}

\subsubsection{Quality versus Quantity} \label{sec:quality vs quantity}
Generally, quantization of data samples decreases the accuracy of any parameter estimation. The quantization of all samples, which is considered so far, may not be the best strategy for decreasing the communication complexity. In fact, if the budget of total number of transmission bits is fixed, then one may want  to sub-sample from the local datasets and allocate the available bits to these samples and discard the rest.

Based on the bound in \eqref{eq:estimator err per symbol} and through an experiment, we show that there is a trade-off between the quality of quantized samples and the size of sub-sampled data to achieve the best performance.

For example, assume that the budget of communication is $K=1000$ bits and the number of local data samples is $n=1000$. This means each local machine can, for instance, transmit $1000$ samples which are quantized by $1$ bit. However, machines can select the first $500$ samples and quantize them to $2$ bits. Which method is better in the sense of minimizing the estimation error $\mathrm{err_{est}}$?

To answer the question, we have simulated our proposed algorithm with $K=1000$ and $n=1000$. In \figurename{}~\ref{fig:quality vs quantity}, the estimation error in \eqref{eq:estimation error} for the estimator in \eqref{eq:rho estimate quantize} is plotted for various bit rates. In this experiment, the true value of correlation is $0.5$. As can be seen from the figure, the error is minimized when $R=4$ bits are used for quantization which is correspond to the sub-sampling size of $250$. The figure also shows that the estimation with large number of highly distorted (quantized) samples is inefficient as well as estimating with low number of high quality samples. In fact, this experiment suggests that in situations with large local datasets and limited communication cost, the optimal strategy is to quantize some portion of the  whole dataset with an acceptable distortion. The upper bound in \eqref{eq:estimator err per symbol} for the per-symbol scheme is also plotted for the sake of comparison.

\begin{figure}
	\centering
	\includegraphics[scale=.5]{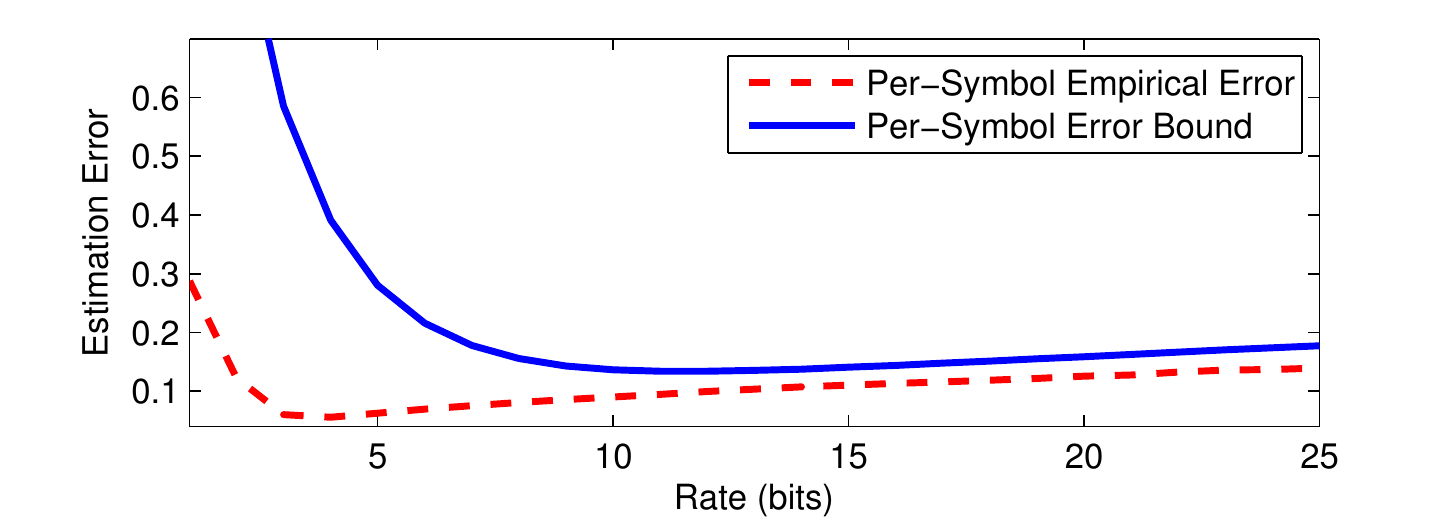}
	\caption{The mean absolute error of correlation estimation. The total communication cost for each machine is limited to $K=1000$ bits.}
	\label{fig:quality vs quantity}
\end{figure}

\subsection{Real-World Dataset: Skeleton Recovering}
To assess our methods on real datasets, the MAD\footnote{Multi-modal Action Database (available at: \url{http://www.humansensing.cs.cmu.edu/mad})} dataset is used. This dataset  is designed for human activity recognition and event detection in the computer vision area \cite{huang2014sequential, han2017space}. The MAD dataset is generated by a Microsoft Kinect sensor in indoor environment. In the dataset, there are $20$ sensors attached to $20$ joints of the human body. Each sensor records $3$D coordinate of its corresponding joint while the subject does an activity. The MAD dataset has three modalities includes RGB video, $3$D depth, and skeleton ($3$D coordinate of the joints). In this experiment, we have used the skeleton modality.

In the dataset, $20$ subjects perform $35$ different actions (e.g. jumping, walking, running, etc.) and each subject repeats the actions twice. Finally, the skeleton dataset contains $243586$ $3$D coordinates per joint.

We assume that the skeleton dataset follows from a tree-structured GGM which its structure is identical to the human body skeleton as depicted in \figurename{}~\ref{fig:human body skeleton}-(a). This assumption intuitively makes sense for such a dataset. Gaussian assumption for similar datasets are proposed in \cite{damianou2013deep} and \cite{lehrmann2013non}.

In this experiment, the skeleton dataset is quantized to several bit rates using the proposed per-symbol quantizer. \figurename{}~\ref{fig:human body skeleton} shows the results for bit rates $1, 3, 5$, and $6$ bits. Applying the Chow-Liu algorithm on $x$ dimension of the original (non-quantized) data, perfectly recovers the body skeleton. Quantizing the $x$ dimension to $1$ bit (using per-symbol and the sign method) results in merely two disagreement edges as showed in \figurename{}~\ref{fig:human body skeleton}-(b). Quantizing to $3$ bits has only one incorrect edge and quantizing to $6$ bits recovers the body skeleton perfectly.

A similar experiment is performed on the $z$ dimension. The results are depicted in \figurename{}~\ref{fig:human body skeleton z}. Interestingly, the $z$ dimension does not follow a tree structured GGM  even for the  case where the original data is available. However, as the bit rate increases the original structure can be recovered reliably.  We have not presented the experiment on the $y$ dimension here. This is due to the fact that the structure inferred from the original data has no relation with the human skeleton.

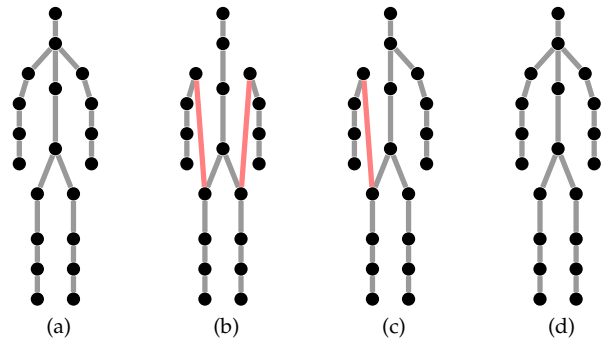
\begin{figure}
	\centering
	\begin{minipage}[b]{.24\linewidth}
		\centering
		\subfloat[]{\input{figures/true-skeleton.tikz}}  
		\label{subfig:true skeleton}
	\end{minipage}		
	\begin{minipage}[b]{.24\linewidth}
		\centering
		\subfloat[]{\input{figures/1-bit-skeleton.tikz}}  
		\label{subfig:1 bit skeleton}
	\end{minipage}		
	\begin{minipage}[b]{.24\linewidth}
		\centering
		\subfloat[]{\input{figures/3-bit-skeleton.tikz}}  
		\label{subfig:3 bit skeleton}
	\end{minipage}
	\begin{minipage}[b]{.24\linewidth}
		\centering
		\subfloat[]{\input{figures/6-bit-skeleton.tikz}}  
		\label{subfig:6 bit skeleton}
	\end{minipage}
	\caption{Structure learning of the human body skeleton on $x$ dimension of the MAD dataset. (a) The true human body skeleton. The estimated structures using quantized data with rates $1$, $3$, and $6$ bits are shown in (b), (c), and (d), respectively. The skeleton of (b) is obtained by both signs and $1$-bit per-symbol methods.}
	\label{fig:human body skeleton}
\end{figure}

\begin{figure}
	\centering
	\begin{minipage}[b]{.24\linewidth}
		\centering
		\subfloat[]{\input{figures/org-skeleton-z.tikz}}  
		\label{subfig:org data z}
	\end{minipage}		
	\begin{minipage}[b]{.24\linewidth}
		\centering
		\subfloat[]{\input{figures/1-bit-skeleton-signs-z.tikz}}  
		\label{subfig:1 bit skeleton signs z}
	\end{minipage}		
	\begin{minipage}[b]{.24\linewidth}
		\centering
		\subfloat[]{\input{figures/1-bit-skeleton-z.tikz}}  
		\label{subfig:1 bit skeleton per-symb z}
	\end{minipage}
	\begin{minipage}[b]{.24\linewidth}
		\centering
		\subfloat[]{\input{figures/7-bit-skeleton-z.tikz}}  
		\label{subfig:6 bit skeleton z}
	\end{minipage}
	\caption{Structure learning of the human body skeleton on $z$ dimension of the MAD dataset. The recovered structure using the original data is shown in (a) . The estimated structures using quantized data with $1$-bit sign method, $1$-bit per-symbol method, and $7$-bit per-symbol are shown in (b), (c), and (d), respectively.}
	\label{fig:human body skeleton z}
\end{figure}
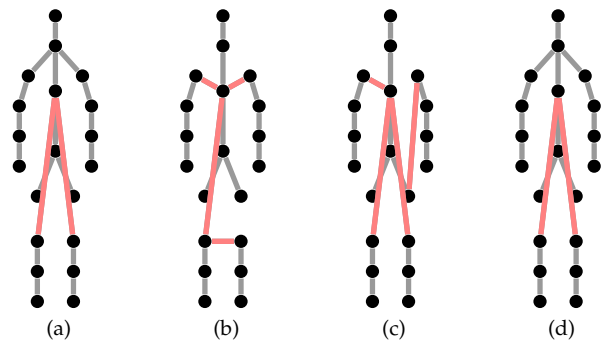

\section{Conclusion}\label{sec:conclusion}
In this paper, we have studied the structure learning of tree-structured GGMs on distributed datasets. Due to communication constraints, we have jointly designed quantization and learning algorithms to achieve high accuracy in inferring the underlying tree structure. In particular, we have proposed two methods for compressing the local datasets: sign method and per-symbol quantizer.

Being rather simple and intuitive, the experimental results show that the per-symbol quantizer yields high accuracy for structure estimation by spending a few bits per sample. Pushing down the number of bits per sample to one, we have proved through experiments and analytical reasoning that even in this case, one can obtain exponentially decaying error probabilities in structure learning.

Our result can be extended in several directions. For instance, the tree structure can be generalized to sparse structures where sparse learning methods such as glasso over the quantized data might be crucial. As another extension, one can study and solve a similar problem on discrete variables with sparse MRFs. Removing the central machine and allowing communication between local machines change the problem significantly and it worths of investigations.


\appendices
\section{Proof of Theorem \ref{thm:distortion}} \label{appendix:proof of theorem 2}
Since variables $x$ and $y$ have unit variances, we have $0 \leq D_1,D_2 \leq 1$.
The relative correlation  error can be upper bounded as follows
\begin{align}
\mathrm{err_{rel}} &=\mathbb{E} \left[\left\vert \frac{1}{n} \sum_{i=1}^{n} x_i y_i - \frac{1}{n} \sum_{i=1}^{n} u_i v_i \right\vert \right] \nonumber\\
&=\mathbb{E} \left[\left\vert \frac{1}{n} \sum_{i=1}^{n}  y_i \left(x_i - u_i\right) + \frac{1}{n} \sum_{i=1}^{n} u_i \left(y_i-v_i\right) \right\vert \right]	\nonumber\\
& \leq \mathbb{E} \left[\left\vert \frac{1}{n} \sum_{i=1}^{n}  y_i \left(x_i - u_i\right) \right\vert \right] + \mathbb{E} \left[\left\vert \frac{1}{n} \sum_{i=1}^{n} u_i \left(y_i-v_i\right) \right\vert \right].	\nonumber
\end{align}
Using the Cauchy-Schwartz inequality, we obtain 
\begin{align}
\mathrm{err_{rel}} & \stackrel{(a)}{\leq} \mathbb{E} \left[\left( \frac{1}{n} \sum_{i=1}^{n}  y_i^2 \right)^{1/2} \left(\frac{1}{n}\sum_{i=1}^{n}\left(x_i - u_i\right)^2 \right)^{1/2} \right]  +\nonumber \\
&\qquad\quad \mathbb{E} \left[\left( \frac{1}{n} \sum_{i=1}^{n}  u_i^2 \right)^{1/2} \left(\frac{1}{n}\sum_{i=1}^{n}\left(y_i - v_i\right)^2 \right)^{1/2} \right]	\nonumber\\
& \stackrel{(b)}{\leq} \left( \mathbb{E} \left[\frac{1}{n} \sum_{i=1}^{n}  y_i^2 \right] \right)^{1/2} \left(\mathbb{E} \left[\frac{1}{n}\sum_{i=1}^{n}\left(x_i - u_i\right)^2 \right] \right)^{1/2} + \nonumber \\
&\qquad\quad \left( \mathbb{E} \left[\frac{1}{n} \sum_{i=1}^{n}  u_i^2 \right] \right)^{1/2} \left(\mathbb{E} \left[\frac{1}{n}\sum_{i=1}^{n}\left(y_i - v_i\right)^2 \right] \right)^{1/2}	\nonumber\\
&\leq \sqrt{\mathbb{E}\left[y^2\right] D_1} + \sqrt{D_2 ~ \mathbb{E} \left[\frac{1}{n} \sum_{i=1}^{n}  u_i^2 \right]}.	\label{eq:distortion upper bound middle}
\end{align}
On the other hand, we have
\begin{align*}
D_1 &\geq \mathbb{E} \left[\frac{1}{n}\sum_{i=1}^{n}\left(x_i - u_i\right)^2 \right] \\
&= 1 + \mathbb{E} \left[\frac{1}{n} \sum_{i=1}^{n}  u_i^2 \right] - 2 \mathbb{E}\left[\frac{1}{n} \sum_{i=1}^{n} x_i u_i\right]	\\
&\geq 1 + \mathbb{E} \left[\frac{1}{n} \sum_{i=1}^{n}  u_i^2 \right] - 2 \mathbb{E}\left[\left(\frac{1}{n}\sum_{i=1}^{n}x_i^2\right)^{1/2} \left(\frac{1}{n}\sum_{i=1}^{n}u_i^2\right)^{1/2}\right]	\\
&\geq 1 + \mathbb{E} \left[\frac{1}{n} \sum_{i=1}^{n}  u_i^2 \right] -  2 \left(\mathbb{E}\left[\frac{1}{n}\sum_{i=1}^{n}x_i^2 \right]  \mathbb{E}\left[\frac{1}{n}\sum_{i=1}^{n}u_i^2 \right] \right)^{1/2}	\\
&= \left(\sqrt{\mathbb{E}\left[\frac{1}{n}\sum_{i=1}^{n}u_i^2 \right]} - 1 \right)^2.
\end{align*}
Hence,
\begin{equation}
\mathbb{E}\left[\frac{1}{n}\sum_{i=1}^{n}u_i^2 \right] \leq \left(\sqrt{D_1} + 1 \right)^2.
\end{equation}
By substituting the above bound into \eqref{eq:distortion upper bound middle}, we obtain
\begin{equation}	\label{eq:distortion upper bound 1}
\mathrm{err_{rel}} \leq \sqrt{D_1} + \sqrt{D_2} + \sqrt{D_1 D_2},
\end{equation}
which completes the proof of Theorem \ref{thm:distortion}.

\ifCLASSOPTIONcompsoc
\else
\fi


\ifCLASSOPTIONcaptionsoff
  \newpage
\fi


\bibliographystyle{IEEEtran}
\bibliography{IEEEabrv,./refs}

%

%
\begin{IEEEbiography}[{\includegraphics[width=1in,height=1.25in,clip,keepaspectratio]{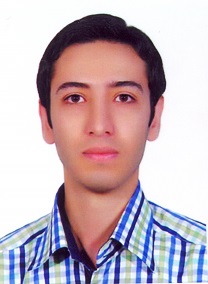}}]{Mostafa Tavassolipour}
	received the B.Sc. degree from Shahed University, Tehran, Iran, in 2009, and the M.Sc. degree from Computer Engineering department of Sharif University of Technology (SUT), Tehran, Iran, in 2011. Currently, He is a Ph.D. student of Artificial Intelligence program at Computer Engineering Department of Sharif University of Technology. His research interests include machine learning, image processing, information theory, content based video analysis, and bioinformatics.
\end{IEEEbiography}
\begin{IEEEbiography}[{\includegraphics[width=1in,height=1.25in,clip,keepaspectratio]{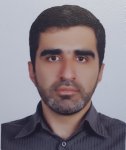}}]{Seyed Abolfazl Motahari}
	is an assistant professor at Computer Engineering Department of Sharif University of Technology (SUT). He received his B.Sc. degree from the Iran University of Science and Technology (IUST), Tehran, in 1999, the M.Sc. degree from Sharif University of Technology, Tehran, in 2001, and the Ph.D. degree from University of Waterloo, Waterloo, Canada, in 2009, all in electrical engineering. From August 2000 to August 2001, he was a Research Scientist with the Advanced Communication Science Research Laboratory, Iran Telecommunication Research Center (ITRC), Tehran. From October 2009 to September 2010, he was a Postdoctoral Fellow with the University of Waterloo, Waterloo. From September 2010 to July 2013, he was a Postdoctoral Fellow with the Department of Electrical Engineering and Computer Sciences, University of California at Berkeley. His research interests include multiuser information theory and Bioinformatics. He received several awards including Natural Science and Engineering Research Council of Canada (NSERC) Post-Doctoral Fellowship.
\end{IEEEbiography}
\begin{IEEEbiography}[{\includegraphics[width=1in,height=1.25in,clip,keepaspectratio]{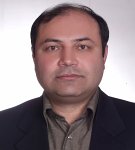}}]{Mohammad Taghi Manzuri Shalmani}
	received the B.Sc. and M.Sc.
	in electrical engineering from Sharif University of
	Technology (SUT), Iran, in 1984 and 1988, respectively.
	He received the Ph.D. degree in electrical and
	computer engineering from the Vienna University
	of Technology, Austria, in 1995. Currently, he is an
	associate professor in the Computer Engineering
	Department, Sharif University of Technology,
	Tehran, Iran. His main research interests include
	digital signal processing, stochastic modeling, and
	Multi-resolution signal processing.
\end{IEEEbiography}




\end{document}

%% file: figures/true-skeleton.tikz
\begin{tikzpicture}[scale=.4, auto,swap]
	\foreach \pos/\name in {{(0, 0)/1},
		{(0, 2)/2},
		{(0, 3.5)/3},
		{(0, 4.5)/4},
		{(.9, 2.5)/5},
		{(1.2, 1.5)/6},
		{(1.2, .5)/7},
		{(1.2, -.5)/8},
		{(-.9, 2.5)/9},
		{(-1.2, 1.5)/10},
		{(-1.2, .5)/11},
		{(-1.2, -.5)/12},
		{(.6, -1.5)/13},
		{(.6, -3)/14},
		{(.6, -4)/15},
		{(.6, -5)/16},
		{(-.6, -1.5)/17},
		{(-.6, -3)/18},
		{(-.6, -4)/19},
		{(-.6, -5)/20}}
	\node[vertex] (\name) at \pos {};
	\foreach \source/ \dest in {1/2, 2/3, 3/4, 3/5, 5/6, 6/7, 7/8,
		3/9, 9/10, 10/11, 11/12, 1/13, 13/14, 14/15,15/16, 1/17, 17/18,18/19, 19/20}
	\path[edge] (\source) -- (\dest);
	
\end{tikzpicture}

%% file: figures/1-bit-skeleton.tikz
\begin{tikzpicture}[scale=.4, auto,swap]
	\foreach \pos/\name in {{(0, 0)/1},
		{(0, 2)/2},
		{(0, 3.5)/3},
		{(0, 4.5)/4},
		{(.9, 2.5)/5},
		{(1.2, 1.5)/6},
		{(1.2, .5)/7},
		{(1.2, -.5)/8},
		{(-.9, 2.5)/9},
		{(-1.2, 1.5)/10},
		{(-1.2, .5)/11},
		{(-1.2, -.5)/12},
		{(.6, -1.5)/13},
		{(.6, -3)/14},
		{(.6, -4)/15},
		{(.6, -5)/16},
		{(-.6, -1.5)/17},
		{(-.6, -3)/18},
		{(-.6, -4)/19},
		{(-.6, -5)/20}}
	\node[vertex] (\name) at \pos {};
	\foreach \source/ \dest in {1/2, 2/3, 3/4, 5/6, 6/7, 7/8,
		9/10, 10/11, 11/12, 1/13, 13/14, 14/15,15/16, 1/17, 17/18,18/19, 19/20}
	\path[edge] (\source) -- (\dest);

	\path[false edge] (17) -- (9);
	\path[false edge] (13) -- (5);
\end{tikzpicture}

%% file: figures/3-bit-skeleton.tikz
\begin{tikzpicture}[scale=.4, auto,swap]
	\foreach \pos/\name in {{(0, 0)/1},
		{(0, 2)/2},
		{(0, 3.5)/3},
		{(0, 4.5)/4},
		{(.9, 2.5)/5},
		{(1.2, 1.5)/6},
		{(1.2, .5)/7},
		{(1.2, -.5)/8},
		{(-.9, 2.5)/9},
		{(-1.2, 1.5)/10},
		{(-1.2, .5)/11},
		{(-1.2, -.5)/12},
		{(.6, -1.5)/13},
		{(.6, -3)/14},
		{(.6, -4)/15},
		{(.6, -5)/16},
		{(-.6, -1.5)/17},
		{(-.6, -3)/18},
		{(-.6, -4)/19},
		{(-.6, -5)/20}}
	\node[vertex] (\name) at \pos {};
	\foreach \source/ \dest in {1/2, 2/3, 3/4, 3/5, 5/6, 6/7, 7/8,
		9/10, 10/11, 11/12, 1/13, 13/14, 14/15,15/16, 1/17, 17/18,18/19, 19/20}
	\path[edge] (\source) -- (\dest);
	
	\path[false edge] (17) -- (9);
	
\end{tikzpicture}

%% file: figures/6-bit-skeleton.tikz
\begin{tikzpicture}[scale=.4, auto,swap]
	\foreach \pos/\name in {{(0, 0)/1},
		{(0, 2)/2},
		{(0, 3.5)/3},
		{(0, 4.5)/4},
		{(.9, 2.5)/5},
		{(1.2, 1.5)/6},
		{(1.2, .5)/7},
		{(1.2, -.5)/8},
		{(-.9, 2.5)/9},
		{(-1.2, 1.5)/10},
		{(-1.2, .5)/11},
		{(-1.2, -.5)/12},
		{(.6, -1.5)/13},
		{(.6, -3)/14},
		{(.6, -4)/15},
		{(.6, -5)/16},
		{(-.6, -1.5)/17},
		{(-.6, -3)/18},
		{(-.6, -4)/19},
		{(-.6, -5)/20}}
	\node[vertex] (\name) at \pos {};
	\foreach \source/ \dest in {1/2, 2/3, 3/4, 3/5, 5/6, 6/7, 7/8,
		3/9, 9/10, 10/11, 11/12, 1/13, 13/14, 14/15,15/16, 1/17, 17/18,18/19, 19/20}
	\path[edge] (\source) -- (\dest);
	
\end{tikzpicture}

%% file: figures/org-skeleton-z.tikz
\begin{tikzpicture}[scale=.4, auto,swap]
	\foreach \pos/\name in {{(0, 0)/1},
		{(0, 2)/2},
		{(0, 3.5)/3},
		{(0, 4.5)/4},
		{(.9, 2.5)/5},
		{(1.2, 1.5)/6},
		{(1.2, .5)/7},
		{(1.2, -.5)/8},
		{(-.9, 2.5)/9},
		{(-1.2, 1.5)/10},
		{(-1.2, .5)/11},
		{(-1.2, -.5)/12},
		{(.6, -1.5)/13},
		{(.6, -3)/14},
		{(.6, -4)/15},
		{(.6, -5)/16},
		{(-.6, -1.5)/17},
		{(-.6, -3)/18},
		{(-.6, -4)/19},
		{(-.6, -5)/20}}
	\node[vertex] (\name) at \pos {};
	\foreach \source/ \dest in {1/2, 3/2, 2/18, 2/14, 3/4, 3/5, 5/6, 6/7, 7/8,
		3/9, 9/10, 10/11, 11/12, 1/13, 14/15,15/16, 1/17,18/19, 19/20}
	\path[edge] (\source) -- (\dest);
	
	\path[false edge] (2) -- (18);
	\path[false edge] (2) -- (14);
	
\end{tikzpicture}

%% file: figures/1-bit-skeleton-signs-z.tikz
\begin{tikzpicture}[scale=.40, auto,swap]
\foreach \pos/\name in {{(0, 0)/1},
	{(0, 2)/2},
	{(0, 3.5)/3},
	{(0, 4.5)/4},
	{(.9, 2.5)/5},
	{(1.2, 1.5)/6},
	{(1.2, .5)/7},
	{(1.2, -.5)/8},
	{(-.9, 2.5)/9},
	{(-1.2, 1.5)/10},
	{(-1.2, .5)/11},
	{(-1.2, -.5)/12},
	{(.6, -1.5)/13},
	{(.6, -3)/14},
	{(.6, -4)/15},
	{(.6, -5)/16},
	{(-.6, -1.5)/17},
	{(-.6, -3)/18},
	{(-.6, -4)/19},
	{(-.6, -5)/20}}
\node[vertex] (\name) at \pos {};
\foreach \source/ \dest in {1/2,3/2, 2/18, 18/14, 3/4, 2/5, 5/6, 6/7, 7/8,
	2/9, 9/10, 10/11, 11/12, 1/13, 14/15,15/16, 1/17,18/19, 19/20}
\path[edge] (\source) -- (\dest);

\path[false edge] (2) -- (18);
\path[false edge] (18) -- (14);
\path[false edge] (5) -- (2);
\path[false edge] (2) -- (9);

\end{tikzpicture}

%% file: figures/1-bit-skeleton-z.tikz
\begin{tikzpicture}[scale=.40, auto,swap]
\foreach \pos/\name in {{(0, 0)/1},
	{(0, 2)/2},
	{(0, 3.5)/3},
	{(0, 4.5)/4},
	{(.9, 2.5)/5},
	{(1.2, 1.5)/6},
	{(1.2, .5)/7},
	{(1.2, -.5)/8},
	{(-.9, 2.5)/9},
	{(-1.2, 1.5)/10},
	{(-1.2, .5)/11},
	{(-1.2, -.5)/12},
	{(.6, -1.5)/13},
	{(.6, -3)/14},
	{(.6, -4)/15},
	{(.6, -5)/16},
	{(-.6, -1.5)/17},
	{(-.6, -3)/18},
	{(-.6, -4)/19},
	{(-.6, -5)/20}}
\node[vertex] (\name) at \pos {};
\foreach \source/ \dest in {1/2,3/2, 2/18, 2/14, 3/4, 13/5, 5/6, 6/7, 7/8,
	2/9, 9/10, 10/11, 11/12, 1/13, 14/15,15/16, 1/17,18/19, 19/20}
\path[edge] (\source) -- (\dest);

\path[false edge] (2) -- (18);
\path[false edge] (2) -- (14);
\path[false edge] (5) -- (13);
\path[false edge] (2) -- (9);

\end{tikzpicture}

%% file: figures/7-bit-skeleton-z.tikz
\begin{tikzpicture}[scale=.4, auto,swap]
	\foreach \pos/\name in {{(0, 0)/1},
		{(0, 2)/2},
		{(0, 3.5)/3},
		{(0, 4.5)/4},
		{(.9, 2.5)/5},
		{(1.2, 1.5)/6},
		{(1.2, .5)/7},
		{(1.2, -.5)/8},
		{(-.9, 2.5)/9},
		{(-1.2, 1.5)/10},
		{(-1.2, .5)/11},
		{(-1.2, -.5)/12},
		{(.6, -1.5)/13},
		{(.6, -3)/14},
		{(.6, -4)/15},
		{(.6, -5)/16},
		{(-.6, -1.5)/17},
		{(-.6, -3)/18},
		{(-.6, -4)/19},
		{(-.6, -5)/20}}
	\node[vertex] (\name) at \pos {};
	\foreach \source/ \dest in {1/2, 3/2, 2/18, 2/14, 3/4, 3/5, 5/6, 6/7, 7/8,
		3/9, 9/10, 10/11, 11/12, 1/13, 14/15,15/16, 1/17,18/19, 19/20}
	\path[edge] (\source) -- (\dest);
	
	\path[false edge] (2) -- (18);
	\path[false edge] (2) -- (14);
	
\end{tikzpicture}